\documentclass[12pt, a4paper]{article}

\usepackage[a4paper]{geometry}
\usepackage{pgfplots}
\pgfplotsset{
    x tick style={color=black},
    y tick style={color=black}
}

\usepackage{url}
\usepackage{listings}
\usepackage{amssymb}
\usepackage{graphicx}
\usepackage{algcompatible}
\usepackage{algorithm}
\usepackage{url}
\usepackage{rotating}
\usepackage{booktabs}
\usepackage{subfig}
\usepackage{amsmath}
\usepackage{bbm}

\renewcommand{\labelenumi}{(\alph{enumi})}
\renewcommand\theenumi\labelenumi

\usepackage[english]{babel}

\usepackage[utf8]{inputenc}
\usepackage{xspace}
\usepackage{amsmath,amsthm,amssymb,mathtools}
\usepackage{lmodern}

\usepackage{cancel}

\usepackage[algo2e,ruled,vlined,linesnumbered]{algorithm2e}

\usepackage{xcolor}
\usepackage{tikz}
\usepackage{graphicx}
\usepackage{soul} % remember to delete

\allowdisplaybreaks[4]
\newtheorem{theorem}{Theorem}
\newtheorem*{theorem*}{Theorem}
\newtheorem{lemma}[theorem]{Lemma}
\newtheorem*{lemma*}{Lemma}

\newtheorem{definition}[theorem]{Definition}
\newtheorem{example}[theorem]{Example}

\newcommand{\NSGA}{\mbox{NSGA-II}\xspace}
\newcommand{\NSGAthree}{\mbox{NSGA-III}\xspace}

\newcommand{\om}{\textsc{OneMax}\xspace}

\newcommand{\omm}{\textsc{OneMinMax}\xspace}
\newcommand{\momm}{$m$\textsc{OneMinMax}\xspace}
\newcommand{\onemax}{\om}

\newcommand{\lotz}{\textsc{LeadingOnesTrailingZeroes}\xspace}
\newcommand{\mlotz}{$m$\textsc{LOTZ}\xspace}

\newcommand{\ojzj}{\textsc{OneJumpZeroJump}\xspace}
\newcommand{\cocz}{\textsc{CountingOnesCountingZeroes}\xspace}
\newcommand{\mcocz}{$m$\textsc{COCZ}\xspace}

\newcommand{\R}{\ensuremath{\mathbb{R}}}

\newcommand{\N}{\ensuremath{\mathbb{N}}} % ohne Null!!!
\newcommand{\Z}{\ensuremath{\mathbb{Z}}}

\DeclareMathOperator{\cDis}{cDis}

\newcommand{\eps}{\varepsilon}

\let\originalleft\left
\let\originalright\right
\renewcommand{\left}{\mathopen{}\mathclose\bgroup\originalleft}
\renewcommand{\right}{\aftergroup\egroup\originalright}

\usepackage{hyperref}

\begin{document}
{\sloppy
\date{}
\title{Runtime Analysis for the NSGA-II:\\ Proving, Quantifying, and Explaining the Inefficiency For Many Objectives}
\author{Weijie Zheng\\School of Computer Science and Technology\\
International Research Institute for Artificial Intelligence\\         
       Harbin Institute of Technology\\
        Shenzhen, China
\and Benjamin Doerr\thanks{Corresponding author.}\\ Laboratoire d'Informatique (LIX)\\ \'Ecole Polytechnique, CNRS\\ Institut Polytechnique de Paris\\ Palaiseau, France
%Benjamin Doerr\\ \'Ecole Polytechnique, CNRS,\\ Laboratoire d'Informatique (LIX),\\ Palaiseau, France
}

\maketitle

%\newpage

%\begin{abstract}
%Despite the wide applications of the NSGA-II and the recent theoretical runtime analyses that show the good efficiency and approximation abilities of the NSGA-II variants, empirical observations about the poor scaling performance of the NSGA-II to more than two objectives are well reported. How to convincingly explain such behaviors is critical for its better usage. In this work, we utilize rigorous mathematical tools to quantify the NSGA-II's difficulty for many objectives. We prove that at least an exponential time both in expectation and with a high probability for the NSGA-II to cover the whole Pareto front for the benchmarks with at least three objectives (and such inefficiency cannot be overcome by increasing the population size), and reveal that the essential reason is the survival selection that can badly uncover the already reached Pareto front points. Our experiments verify our findings.  
%\end{abstract}

\begin{abstract}
The NSGA-II is one of the most prominent algorithms to solve multi-objective optimization problems. Despite numerous successful applications, several studies have shown that the NSGA-II is less effective for larger numbers of objectives. In this work, we use mathematical runtime analyses to rigorously demonstrate and quantify this phenomenon. We show that even on the simple $m$-objective generalization of the discrete OneMinMax benchmark, where every solution is Pareto optimal, the NSGA-II also with large population sizes cannot compute the full Pareto front (objective vectors of all Pareto optima) in sub-exponential time when the number of objectives is at least three. %Our proofs suggest that 
The reason for this unexpected behavior lies in the fact that in the computation of the crowding distance, the different objectives are regarded independently. This is not a problem for two objectives, where any sorting of a pair-wise incomparable set of solutions according to one objective is also such a sorting according to the other objective (in the inverse order).
\end{abstract}

\section{Introduction}\label{sec:intro}
{M}{any} real-world optimization problems have several, usually conflicting objectives. In this situation, it is not possible to simply compute one good solution. Instead, a common solution concept for such multi-objective problems is to compute a representative set of Pareto solutions (solutions which are not strictly dominated by others) and let a human decision maker choose one of these. The two dominant approaches to such problems are mathematical programming and evolutionary algorithms (EAs). The latter profits from the fact that evolutionary algorithms naturally work with sets of solutions (``populations''). 

The most prominent multi-objective evolutionary algorithm (MOEA) with many successful applications in various domains is the \emph{non-dominated sorting genetic algorithm II} (\NSGA)~\cite{DebPAM02}, see~\cite{ZhouQLZSZ11} or the more than 40,000 citations on Google scholar. Despite numerous positive results on the \NSGA, several studies have observed that the \NSGA is less effective when the number of objectives grows. From their experimental comparison of three MOEAs, Khare, Yao, and Deb~\cite{KhareYD03} reported that the \NSGA increasingly suffers in converging to the global Pareto front when the number of objectives increases from $2$ to $8$. This early study is a pure performance comparison, so no explanations for the different behaviors were sought. Also, it should be noted that the other algorithms regarded also suffered in different ways from growing numbers of objectives. Purshouse and Fleming~\cite{PurshouseF07} aimed for further empirical observations. They studied the range of suitable configurations of the NSGA-II that can result in good solutions. That the suitable ranges shrink along with the increasing number of objectives, indicates the increasing difficulty for the NSGA-II for more objectives. They also observed some factors that might be relevant for the poor performance of the NSGA-II, like that when the population evolves, the proportion of non-dominated solutions in the population rapidly increases to $100\%$, and that in each iteration, only a low proportion of the newly generated solutions dominates the current solutions. These observations are regarded as one of the key challenges for many-objective optimizations, see the surveys~\cite{IshibuchiTN08,LiLTY15}. 

Since apparently there is a lack of understanding of the performance of the \NSGA for many-objective problems, we try to approach this research question via a mathematical runtime analysis~\cite{DrosteJW02}, see also\cite{NeumannW10,AugerD11,Jansen13,DoerrN20}. Such analyses are an integral part of the theory of heuristic search. While often restricted to simple algorithmic settings, this alternative approach has led to several deep and very reliable (namely mathematically proven) results in the past, see, e.g.,~\cite{Crawford19,CorusOY21tec,DangEL21aaai,NeumannW22}. Also, often the proofs also reveal the reason why a certain phenomenon can be observed.

To this aim, we conduct a mathematical runtime analysis of the \NSGA on the $m$-objective version of the classic bi-objective \omm benchmark. This pseudo-Boolean (that is, defined on bit-strings of length~$n$) benchmark is very simple in several respects, for example, any solution is Pareto optimal and the objectives are all equivalent to the \onemax benchmark, which is generally considered as the easiest single-objective Pseudo-boolean benchmark~\cite{DoerrJW12algo,Sudholt13,Witt13}. For the bi-objective \omm problem, a good performance of the \NSGA has been proven recently~\cite{ZhengLD22,BianQ22,ZhengD22gecco}. 

Our main (proven) result is that -- despite the simplicity of the problem and in drastic contrast to the bi-objective setting --  for all numbers $m \ge 3$ of objectives the \NSGA also with large population sizes cannot find the full Pareto front faster than in exponential time. Even worse, we prove that for an exponential time, the population of the \NSGA will miss a constant fraction of the Pareto front. Our experiments confirm this finding in a very clear manner. 

%Our mathematical runtime analysis also 
We also give an explanation for the drastic change of behavior between two and three objectives, namely that in the definition of the crowding distance the different objectives are regarded independently (we make this point much more precise in the body of this paper when all necessary notation is introduced). This is not a problem for two objectives, because for a set of pair-wise non-dominated solutions a sorting with respect to one objective automatically is a sorting with respect to the other objective (in the opposite order). Hence here in fact the two objectives automatically are not treated independently. From three objectives on, such a correlation between the objectives does not exist, and this can lead to the problems made precise in this work. 

This understanding has two implications. On the negative side, it appears very likely that the difficulties observed for the \omm benchmark will also occur for many other optimization problems, including continuous optimization problems. On the positive side, this understanding suggests to search for an alternative crowding distance measure that does not treat the objectives independently. 

We note that given the known difficulties of the \NSGA with more objectives, several alternatives to the crowding distance have been proposed. For example, Deb and Jain~\cite{DebJ14} proposed the \NSGAthree, where the crowding distance is replaced by a system of reference points. In their SMS-EMOA, Beume, Naujoks, and Emmerich~\cite{BeumeNE07} replaced the crowding distance by the hypervolume. The hypervolume measure was also considered as the second sorting criterion in the multiobjective CMA-ES~\cite{IgelHR07}. It was also used in hybrid with the crowding distance~\cite{WangEDB19}. Instead of removing points with the smallest crowding distances in the \NSGA, the diversity preservation way via clustering from SPEA~\cite{ZitzlerT99} was used to select the survived individuals~\cite{DebMM03,VachhaniDP16}. However, none of these alternatives is as accepted in practice as the \NSGA. For this reason, we hope that our work can not only spur the development of superior algorithms, but also motivate practitioners to try moving from the \NSGA to these more modern algorithms. We note that very recently~\cite{WiethegerD23} a runtime analysis of the \NSGAthree appeared that showed that the difficulties proven here for the \NSGA do not arise with the \NSGAthree.

We further note that several promising variations of the crowding distance have been proposed, e.g.,~\cite{KukkonenD06,FortinP13}, which again are not used a lot in practice. For the latter~\cite{FortinP13}, we easily observe that the \NSGA with this variant of the crowding distance can effectively optimize the three-objective \omm benchmark. However, since this variant only differs from the traditional crowding distance when objective values occur multiple times in the population, we also immediately see that this variant displays the same unfavorable behavior on Example~\ref{ex:bad} as the traditional \NSGA.

\section{Previous Work}

For reasons of brevity, we will not discuss the practical and empirical works on the \NSGA beyond what we did in the introduction. For the theoretical side, we note that mathematical analyses
%\footnote{Note: Theoretical/mathematical/runtime analyses here mean the analyses on the number of function evaluations for a black-box algorithm to reach a predefined goal. The complexity of the implementation in one iteration of the algorithm is a different topic.} 
have always accompanied the design and analysis of evolutionary algorithms. The first (mathematical) runtime analyses of MOEAs have appeared in the early 2000s~\cite{LaumannsTZWD02,Giel03,Thierens03}. They and many of the subsequent works regarded synthetic algorithms like the SEMO or global SEMO (GSEMO). While these are much simpler than the algorithms used in practice, they are still close enough to these to admit useful conclusions. It took some time until more realistic MOEAs such as the hypervolume-based \emph{Simple Indicator-Based Evolutionary Algorithm (SIBEA)}~\cite{BrockhoffFN08} or the decomposition-based MOEA/D~\cite{LiZZZ16} were regarded under the runtime analysis paradigm, and only very recently the first runtime analysis of the \NSGA was presented~\cite{ZhengLD22,ZhengD23aij}.

Besides demonstrating that runtime analyses are feasible for this kind of complex algorithms, this work proved that the \NSGA with population size $N$ computes the full Pareto front of the biobjective \omm benchmark defined on bit strings of length $n$ in expected time $O(Nn\log n)$ if $N \ge 4(n+1)$. For the \lotz benchmark, a bound of $O(Nn^2)$ was shown. For the reasonable choice $N = \Theta(n)$ of the population size, these runtimes agree with previous results for the SEMO and GSEMO algorithms. Using a population size strictly larger than the size $n+1$ of the Pareto front is necessary: For $N = n+1$, the \NSGA will not find the Pareto front for an exponential time, and moreover, miss a constant fraction of it. For smaller population sizes, however, the \NSGA can still compute good approximations of the Pareto front as shown in~\cite{ZhengD22gecco}. These two works regard a simplified version of the \NSGA that does not employ crossover. In~\cite{BianQ22}, for the first time a runtime analysis for the \NSGA with crossover is conducted, however, the runtime bounds proven are not lower than those in~\cite{ZhengLD22}. Also, this work lowers the required population size to $2(n+1)$ by assuming that two objectives are sorted inversely. Further, runtime improvements are obtained from a novel selection mechanism. \cite{DoerrQ23tec} is the first work to conduct a runtime analysis on a multimodal benchmark, the \ojzj benchmark~\cite{ZhengD23ecj}. The first lower bounds, matching the previously shown upper bounds for \omm and \ojzj, were shown in~\cite{DoerrQ23LB}. That crossover can lead to speed-ups, was proven in~\cite{DangOSS23aaai,DoerrQ23crossover}. A first runtime analysis for the \NSGA on a combinatorial optimization problem, was conducted in~\cite{CerfDHKW23}. We note that all these works regard problems consisting of two objectives only. 

%\subsection{Runtime Results of MOEAs for Many Objectives}
There are some runtime analyses for many-objective problems, however mostly for the very simplistic SEMO algorithm. Already the journal version~\cite[Section~V]{LaumannsTZ04} of the first MOEA runtime analysis~\cite{LaumannsTZWD02} contains a many-objective runtime analysis, namely a proof that the SEMO covers the Pareto front of \mcocz and \mlotz (which are the $m$-objective analogues of the classic \cocz and \lotz benchmarks) with problem size $n$ and objective number $m\ge 4$ in an expected number of $O(n^{m+1})$ function evaluations. %Some other variants' runtime results are also given. 
Some of these bounds were improved in~\cite{BianQT18ijcaigeneral}, namely to $O(n^m)$ for the SEMO on \mcocz with $m > 4$ and to $O(n^3\log n)$ for the special case $m=4$. As often in the runtime analysis of MOEAs, the complicated population dynamic prevented the proofs of any interesting lower bounds. In the only other runtime analysis for many-objective problems, Huang, Zhou, Luo, and Lin~\cite{HuangZLL21} analyzed how the MOEA/D optimizes the benchmarks \mcocz and \mlotz. As the MOEA/D decomposes the multi-objective problem into several single-objective subproblems and solves these in a co-evolutionary way, this framework is fundamentally different from most MOEAs and in particular the \NSGA, so we do not discuss these results in more detail.

\section{Preliminaries}\label{sec:pre}
In this section, we give brief introductions to many-objective optimization and to two kinds of MOEAs (NSGA-II series and SEMO series).

\subsection{Many-Objective Optimization}\label{ssec:maoo}
We call optimization problems with more than two objectives \emph{many-objective optimization} problems. In this paper, we consider the maximization of a pseudo-Boolean many-objective optimization problem with $m\ge 3$ objectives and problem size of $n$, $f=(f_1,\dots,f_m):\{0,1\}^n\rightarrow \R^m$. Dominance is used to compare the quality of solutions. We say that $x$ \emph{weakly dominates} $y$, denoted as $x \succeq y$, if and only if $f_i(x)\ge f_i(y)$ for all $i=1,\dots,m$, and we say that $x$ \emph{(strictly) dominates} $y$, denoted as $x \succ y$, if and only if $x \succeq y$ and there exists $i\in\{1,\dots,m\}$ such that $f_i(x) > f_i(y)$. We call a solution $x$ \emph{Pareto optimum} if there is no solution that strictly dominates it. The function value of the Pareto optimum $x$ is called a \emph{Pareto front point}. All Pareto front points form the \emph{Pareto front}, denoted by $F^*$ in this paper. We call a Pareto front point $v\in F^*$ \emph{covered} by a (multi-)set of solutions $P$ if there exists $x\in P$ such that $v=f(x)$. 

In this paper, the \emph{aim} is to compute a set of solutions $P$ that covers the Pareto front, that is, such that for any Pareto front point $v\in F^*$ there exists $x\in P$ such that $v=f(x)$. This aim is widely used in runtime analysis works~\cite{LaumannsTZ04,BianQT18ijcaigeneral}. We note that this aim is different from obtaining all Pareto optima (the entire efficient set), the target in multimodal optimization~\cite{Preuss15,LiangYQ16,TanabeI20}.

%\merk{WJ: for the descriptions, use $m'=m/2$ for the number of blocks}
%\begin{definition}[\mcocz~\cite{LaumannsTZ04}]
%Let $m,n\in \N$, $m$ be even, and $n'=n/m\in N$. The $m$-objective function \mcocz$=(f_1,\dots,f_m)$: $\{0,1\}^n \rightarrow \R^m$ of a solution $x=(x_1,\dots, x_n)\in\{0,1\}^n$ is defined as
%\begin{align*}
%f_i=\sum_{j=1}^{n/2}x_j+
%\begin{cases}
%\sum_{j=1}^{n'}x_{j+n/2+(i-1)n'/2}, &\text{if $i$ is odd;}\\
%\sum_{j=1}^{n'}(1-x_{j+n/2+(i-2)n'/2}), &\text{else.}
%\end{cases}
%\end{align*} 
%\end{definition}
%The Pareto optima have the form of $(1^{n/2}*)$ with $*\in\{0,1\}^{n/2}$, and the Pareto front is $\{n/2+(i_1,n'-i_1,\dots,i_{m/2},n'-i_{m/2}) \mid i_1,\dots,i_{m/2} \in [0..n']\}$. Let $M$ be the size of the Pareto front, then $M=(n'+1)^{m/2}=(n/m+1)^{m/2}$.
%
%
%
%
%\begin{definition}[\mlotz~\cite{LaumannsTZ04}]
%Let $m,n\in \N$, $m$ be even, and $n'=2n/m\in \N$. The $m$-objective function \mlotz$=(f_1,\dots,f_m)$: $\{0,1\}^n \rightarrow \R^m$  is defined by
%\begin{align*}
%f_i=
%\begin{cases}
%\sum_{j=1}^{n'}\prod_{k=1}^j x_{k+(i-1)n'/2}, &\text{if $i$ is odd;}\\
%\sum_{j=1}^{n'}\prod_{k=j}^{n'}(1-x_{k+(i-1)n'/2}), &\text{else}
%\end{cases}
%\end{align*} 
%for all $x=(x_1,\dots, x_n)\in\{0,1\}^n$.
%\end{definition}
%The Pareto optima have the form of $(1^{i_1}0^{n'-i_1}\dots1^{i_{m/2}}0^{n'-i_{m/2}})$ with $i_1,\dots,i_{m/2} \in [0..n']$, and the Pareto front is $\{(i_1,n'-i_1,\dots,i_{m/2},n'-i_{m/2}) \mid i_1,\dots,i_{m/2} \in [0..n']\}$. Let $M$ be the size of the Pareto front, then $M=(n'+1)^{m/2}=(2n/m+1)^{m/2}$.

\subsection{Algorithms}
For MOEAs, an individual is a solution, that is, a bit-string $x\in\{0,1\}^n$, and a population is a multi-set of individuals. The \emph{NSGA-II} algorithm~\cite{DebPAM02} has a fixed population size $N$. After the random initialization, in each iteration (also called ``generation''), the offspring population $Q_t$ with size $N$ is generated from the current population (parent population) $P_t$ via parent selection and usually crossover and mutation. We call the parent selection strategy \emph{fair selection} if each individual in the parent population is selected once as a parent to generate an offspring individual. Other parent selection strategies include binary tournament selection, stochastic tournament selection~\cite{BianQ22}, etc. To apply crossover, the \NSGA selects $N$ parents and divides them into $N/2$ pairs. In~\cite{DebPAM02}, with probability $0.1$, the parents go directly to the mutation, and otherwise, the algorithm uses the following \emph{one-point crossover}. Independently in each pair, it uniformly at random picks a crossover position $i\in[1..n]$, and two parents exchange their first $i$ bits to form two new parents for the mutation. For the mutation, the \emph{one-bit mutation} operator generates an offspring by flipping one (uniformly at random picked) bit of its parent, and the \emph{standard bit-wise mutation} operator generates an offspring by flipping each bit independently with probability $1/n$.

After the offspring population is generated, $N$ individuals are kept among from the combined parent and offspring population $R_t=P_t\cup Q_t$ of size $2N$. The NSGA-II first uses the procedure fast-non-dominated-sort() from~\cite{DebPAM02} to partition the population into fronts $F_1,F_2,\dots,$ where $F_i$ consists of the non-dominated individuals in $R_t\setminus \bigcup_{j<i} F_j$. Let $i^*=\min\{i\ge 1\mid |\bigcup_{j=1}^i F_j |\ge N\}$. Then all individuals in $F_i$ with $i<i^*$ are kept in the next parent population. For the individuals in the critical $F_{i^*}$, it computes their crowding distance (see Algorithm~\ref{alg:cDis}) and keeps the $N-|\bigcup_{j<i^*} F_j|$ individuals with largest crowding distance (breaking ties randomly). This original survival selection has some drawbacks and a modified (sequential) one was proposed~\cite{KukkonenD06} and proven to be more efficient for the approximation for the bi-objective \omm problem~\cite{ZhengD22gecco}. Algorithm~\ref{alg:nsgaii} describes the procedure of the NSGA-II with the original survival selection and also with the sequential survival selection. 
\begin{algorithm}[!ht]
    \caption{Computation of the crowding distance $\cDis(S)$}
    \textbf{Input:} $S=\{S_1,\dots,S_{|S|}\}$, a set of individuals\\
    \textbf{Output:} $\cDis(S)=(\cDis(S_1),\dots,\cDis(S_{|S|}))$, where $\cDis\left(S_i\right)$ is the crowding distance for $S_i$
    \begin{algorithmic}[1]
    \STATE $\cDis(S)=(0,\dots,0)$
    \FOR {each objective $f_i$}\label{stp:ccDiss}
    \STATE {Sort $S$ in order of descending $f_i$ value: $S_{i.1},\dots,S_{i.{|S|}}$}\label{stp:sort}
    \STATE {$\cDis\left(S_{i.1}\right)=+\infty, \cDis\left(S_{i.{|S|}}\right)=+\infty$}
    \FOR {$j=2,\dots, |S|-1$}
    \STATE {$\cDis(S_{i.j})=\cDis(S_{i.j}) + \frac{f_i\left(S_{i.{j-1}}\right)-f_i\left(S_{i.{j+1}}\right)}{f_i\left(S_{i.1}\right)-f_i\left(S_{i.{|S|}}\right)}$}\label{stp:comcDis}
    \ENDFOR
    \ENDFOR\label{stp:ccDise}
    \end{algorithmic}
    \label{alg:cDis}
\end{algorithm}

\begin{algorithm}[!ht]
    \caption{NSGA-II}
    \begin{algorithmic}[1]
    \STATE {Uniformly at random generate the initial population $P_0=\{x_1,x_2,\dots,x_N\}$ with $x_i\in\{0,1\}^n,i=1,2,\dots,N.$}\label{ste:initialize}
    \FOR{$t = 0, 1, 2, \dots$} \label{ste:iterate}
    \STATE {Generate the offspring population $Q_t$ with size $N$}\label{ste:generate}
    \STATE {Use fast-non-dominated-sort() in~\cite{DebPAM02} 
    to divide $R_t$ into fronts $F_1,F_2,\dots$}
    \label{ste:sort}
    \STATE {Find $i^* \ge 1$ such that $|\bigcup_{i=1}^{i^*-1}F_i| < N$ and $|\bigcup_{i=1}^{i^*}F_i| \ge N$}\label{ste:rank}
    \STATE {Use Algorithm~\ref{alg:cDis} to separately calculate the crowding distance of each individual in $F_{1},\dots,F_{i^*}$}\label{ste:cDis}
    \IF {Original Survival Selection~\cite{DebPAM02}} 
    \STATE {Let $\tilde{F}_{i^{*}}$ be the $N-|\bigcup_{i=1}^{i^*-1}F_{i}|$ individuals in $F_{i^*}$ with largest crowding distance, chosen at random in case of a tie}\label{ste:final front}
    \STATE {$P_{t+1}=\left(\bigcup_{i=1}^{i^*-1}F_i\right)\cup\tilde{F}_{i^*}$}\label{ste:new parents}
    \ELSIF {Sequential Survival Selection~\cite{KukkonenD06,ZhengD22gecco}}   
    \WHILE{$|\bigcup_{i=1}^{i^*}F_{i}|\neq N$}\label{ste:flybegin}
    \STATE {Let $x$ be the individual with the smallest crowding distance in $F_{i^*}$, chosen at random in case of a tie}
    \STATE {$F_{i^*} = F_{i^*} \setminus \{x\}$}
    \STATE {Update the crowding distance of all neighbors of $x$}
    \ENDWHILE\label{ste:flyend}
    \ENDIF
    \ENDFOR 
    \end{algorithmic}
    \label{alg:nsgaii}
\end{algorithm}

The SEMO and GSEMO (see Algorithm~\ref{alg:semo}) are MOEAs predominantly analyzed in the evolutionary computation theory community.  Different from the \NSGA, they do not work with a fixed population size, but keep any solution until a better one (in the domination sense) is found. They start with a single random solution. Each iteration, one offspring is generated from mutating a random member of the population. It is added to the population if it is not strictly dominated by a member of the population, and once it is added, the individuals that are weakly dominated by it will be removed. The only difference between SEMO and GSEMO is that SEMO uses one-bit mutation and GSEMO uses standard bit-wise mutation.

\begin{algorithm}[!ht]
    \caption{SEMO/GSEMO}
    \begin{algorithmic}[1]
    \STATE {Uniformly at random generate an individual (solution) $x\in\{0,1\}^n$, and the initial population $P_0=\{x\}$}
    \FOR{$t = 0, 1, 2, \dots$}
    \STATE {Uniformly at random select an individual $x$ from $P_t$}
    \STATE {Generate the offspring $x'$ by one-bit mutation for the SEMO, or by standard bit-wise mutation for the GSEMO}
    \IF {there is $y\in P_t$ such that $x'\succeq y$} 
    \STATE {$P_{t+1}=\{y\in P_t \mid x' \nsucceq y\} \cup \{x'\}$}
    \ENDIF
    \ENDFOR 
    \end{algorithmic}
    \label{alg:semo}
\end{algorithm}

\section{Ineffectiveness for Four and More Objectives}\label{sec:4omm}

In this section, we prove our main result that the \NSGA cannot effectively compute the Pareto front of the \momm benchmark (except for the case $m=2$). Since $m$ is necessarily even in the definition of \momm below, we discuss separately in the subsequent section an example of a three-objective \omm problem that also cannot be solved effectively by the \NSGA. This shows that the bi-objective setting is structurally different and, in a sense, an exceptional case. 

We start by giving the formal definition of the $m$-objective version of the \omm problem, then give an informal description of the reasons for the difficulties of the \NSGA in many-objective optimization, and finally prove our main result that the \NSGA selection with very high probability loses a constant fraction of the Pareto front (when the population size is larger than the Pareto front size by a constant factor, which can be arbitrarily large). This result then implies that the \NSGA takes at least a time exponential in the Pareto front size to compute the full Pareto front.

\subsection{The \momm Benchmark}

We now define an $m$-objective version of the bi-objective \omm benchmark. The \omm benchmark, first proposed in~\cite{GielL10}, is arguably the easiest multi-objective benchmark. Different from the still simple \cocz and \lotz benchmarks, it has the property that any solution is Pareto optimal. We now define an $m$-objective version of it maintaining this property. Recalling the widely used aim in the runtime analysis discussed in Section~\ref{ssec:maoo}, this paper profits from focussing on the process of covering  the Pareto front. Despite the simplicity of the many-objective \omm problem, in this section, we will show that the \NSGA cannot efficiently optimize it. 

To define an $m$-objective version of this benchmark, we proceed in the same fashion that was used in~\cite{LaumannsTZ04} to define the many-objective versions of the \cocz and \lotz benchmarks, that is, for an even number $m$, we split the $n$ binary decision variables into $m/2$ blocks of size $n/(m/2)$ and define a bi-objective \omm problem on each block.   

\begin{definition}[\momm]
Let the number $m \in \N$ of objectives be even and let the problem size $n$ be a multiple of $m/2$. Let $n'=2n/m\in \N$. The $m$-objective function \momm$=(f_1,\dots,f_m)$: $\{0,1\}^n \rightarrow \R^m$  is defined by
\begin{align*}
f_i(x)=
\begin{cases}
\sum_{j=1}^{n'}(1-x_{j+(i-1)n'/2}),&\text{if $i$ is odd;}\\
\sum_{j=1}^{n'}x_{j+(i-2)n'/2},  &\text{else}
\end{cases}
\end{align*} 
for all $x=(x_1,\dots, x_n)\in\{0,1\}^n$.
\end{definition}

We note that, as for the bi-objective \omm benchmark, any $x \in \{0,1\}^n$ is a Pareto optimum of \momm. The Pareto front thus is $F^* = \{(i_1,n'-i_1,\dots,i_{m/2},n'-i_{m/2}) \mid i_1,\dots,i_{m/2} \in [0..n']\}$. Its size is $M := |F^*| = (n'+1)^{m/2}=(2n/m+1)^{m/2}$. Note that the Pareto front size is only polynomial in the problem size $n$ when $m$ is a constant. Naturally, the set of all Pareto optima (efficient set) is exponential in $n$.

\subsection{Deficiencies of the Many-Objective Crowding Distance}\label{ssec:exp}

By definition, the crowding distance of an individual $x$ in a set of pair-wise non-dominated individuals $S$ is computed as the sum of the crowding distance contributions of $x$ in each objective. This is convenient as it allows to compute the crowding distance by solving $m$ sorting problems. Also, this appears intuitively a good measure of the importance of a solution, at least when regarding bi-objective illustrations such as Figure~1 in~\cite{DebPAM02}. Unfortunately, and this is a main insight from this work,  this intuition is correct only for two objectives, and for larger numbers the independent treatment of the objectives allows for very undesirable results. 

\begin{example}\label{ex:bad}
Let us give a simple example.  
%Let $f : [0..200]^4 \to \R$ be the $4$-objective \omm problem with problem\change{\st{s}} size $n = 200$. 
%Let $S_0$ be the set of eight individuals such that $f(S)$ is the set of the eight objective values $(a,n-a,b,n-b)$ such that one of $a$ and $b$ is in $\{99,101\}$ and the other is in $\{0,200\}$. In other words, $f(S) = \{(99,101), (101,99)\} \times  \cup \{(0,200), (200,0\} \times \{(99,101), (101,99)\}$. Let $x = \{(100,100,100,100)\}$ and $S = S_0 \cup \{x\}$. One would hope that $x$ has a large crowding distance in $S$ since $f(x)$ it is far from all other objective values. Unfortunately, this is not true. Since for each of the four objectives the values $99$ and $101$ occur in $S_0$, the crowding distance of $x$ in $S$ is only $4 \cdot \frac 2 {200} =  
Let $f : \{0,1\}^{400} \to [0..200]^4$ be the $4$-objective \momm problem with problem size $n = 400$ and $m = 4$. 
Let $S$ be a set of $5$ individuals such that 
\begin{align*}
f(S) = \{&({}99,101,0,200),(101,99,0,200),\\
&{}(0,200,99,101),(0,200,101,99),\\
&{}(100,100,100,100)\}.
\end{align*}
Let $x \in S$ be the individual with $f(x) = (100,100,100,100)$. One would hope that $x$ has a large crowding distance since it is relatively isolated in many respects. For example,  $f(x)$ is the only point in the half-plane $\{(a,b,c,d) \in \R^4 \mid a+c > 101\}$, which contains roughly $\frac 78$ of $[0..200]^4$. Also, with an $L_1$ distance of more than $200$ the point $f(x)$ is far from all other points, which in contrast have an $L_1$ distance of only $2$ from one other point in $f(S)$. Unfortunately, this hope does not come true. The crowding distance of $x$ is $4 \cdot \frac 2 {101} = \frac{8}{101} \le 0.08$. In contrast, all other points in each objective have a crowding distance contribution of at least $\frac{99}{101} \ge 0.98$. Hence from the crowding distance perspective, $x$ appears as the by far least important point, where in contrast it is obvious that one could easily omit one of $(99,101,0,200)$ and $(101,99,0,200)$ from $S$ without noticeably reducing the diversity of~$S$.
\end{example}

The reason for this undesired result is that the computation of the crowding distance is based on an independent consideration of the objectives. This allows that points far away from a solution $x$ have a huge influence on its crowding distance.
For example, the points $(99,101,0,200)$ and $(101,99,0,200)$ are both far from $(100,100,100,100)$ in the objective space, yet they are the reason for a small crowding distance contribution of the first and second objective to the crowding distance of $(100,100,100,100)$. Likewise, the points $(0,200,99,101)$ and $(0,200,101,99)$ are far from $(100,100,100,100)$, but cause a small crowding distance contribution of the third and fourth objective.

We note that this problem does not occur for two objectives. The reason is that there a sorting of a domination-free set with respect to one objective automatically is a sorting with respect to the other objective (in inverse order). Consequently (when assuming distinct objective values), the crowding distance of an individual $x$ is determined only by its two unique neighbors, which are the same for both objectives. Hence it cannot happen here that ``points far away'' have an influence on the crowding distance of~$x$. 

We also note that the above example also holds for continuous problems. For example, these points are in the critical front (thus are mutually non-dominated) for solving a continuous problem, and the point $x$ in the less crowded regions only has a small crowding distance.

\subsection{An Exponential Lower Bound for $m\ge 4$}
We described above in some detail what is the true reason for the difficulties we observe in this section, and we did so because we feel that it aids the reader's understanding and might ease finding a remedy to this problem. 

Fortunately, for our mathematical proofs, we can resort to a simpler, elementary-algebra argument, which will imply that there is only a very limited number of individuals with positive crowding distance. 
We note that a similar argument was used in~\cite{ZhengLD22} in a positive manner (in the analysis of biobjective problems). There, the existence of a sufficient number of individuals with crowding distance zero was exploited to argue that these individuals will be removed first, and consequently, the (at least one) individual with positive crowding distance per Pareto front point will survive the selection process.
%In a similar fashion as done in~\cite{ZhengLD22} for the bi-objective \omm function, we prove an upper bound on the number of individuals with positive (including infinite) crowding distance. Different from there, here (for $m \ge 4$) we see that the number of individuals with positive crowding distance is much smaller than the size of the Pareto front. This will turn out to be a major obstacle for the \NSGA to compute the full Pareto front even with large populaton sizes.

\begin{lemma}
Let $m\in \N$ with $m \ge 4$. Let $S$ be a set of pair-wise non-dominated individuals in $\{0,1\}^n$. Assume that we compute the crowding distance $\cDis(S)$ with respect to the objective function \momm via Algorithm~\ref{alg:cDis}. Then at most $4n+2m$ individuals in $S$ have a positive crowding distance.
\label{lem:poscDis}
\end{lemma}

\begin{proof}
Consider the sorted list $S_{i.1},\dots,S_{i.{|S|}}$ (Step~\ref{stp:ccDise} in Algorithm~\ref{alg:cDis}) for calculating the component of the crowding distance w.r.t. a certain objective $f_i$. Let $V=f_i(S)$ and $V=\{v_1,\dots,v_{|V|}\}$ such that $v_1>v_2>\dots>v_{|V|}$. For $j \in [1..|V|]$, let $H_j=\{s\in S\mid f_i(s)=v_j\}$. Since $S$ is sorted according to increasing $f_i$ value, there are $a, b \in [1..|S|]$ with $a \le b$ such that $H_j = \{S_{i.a}, \dots, S_{i.b}\}$. By definition of the crowding distance, the at most $2$ individuals $S_{i.a}$ and $S_{i.b}$ are the only ones in $H_j$ to have a positive crowding distance contribution from the $i$-th objective.
%Let first $j\in\{1,|V|\}$. 
%If $|H_j|=1$, then $s\in H_j$ will be the first or the last point of the sorted list and will be assigned with infinite crowding distance; if $|H_j| >1$, then one item in $H_j$ will be the first or the last point of the sorted list, which will be assigned with infinite crowding distance, and one item will have its left or right item in the list from $H_{j'}$ with $j'\ne j$, which will be assigned with positive crowding distance. Similarly, for each $v_j, i\in [2..|V|-1]$, $|H_j|=1$ will result in that the item in $H_j$ has both its left and right items in the list from different $H_{j'}$ and $H_{j''}$ ($j',j''\ne j$), and thus has positive crowding distance. $|H_j| >1$ will result in two such items in $|H_j|$ with positive crowding distance. In summary, at most two corresponding individuals for one certain $v_j\in V$ have infinite or positive crowding distance. 
Since $|V| \le n'+1$ and there are $m$ objectives, we know that there are at most $2m(n'+1)=2m(2n/m+1)=4n+2m$ individuals with positive crowding distance. 
\end{proof}

We note that the result above only refers to the computation of the crowding distance. It is thus independent of other components of the \NSGA like the mutation operator, the possible use of crossover, and the selection rules. Also, due to its simplicity, it can easily be transformed to apply to other objective functions with a limited number of values in each objective, e.g., the knapsack problem, for which several many-objective works exist (see~\cite{IshibuchiTMN14} and the references therein). In its most general form, it would state that if $f$ is an $m$-objective problem and $\nu_i$ is the number of objective values of the $i$-th objective, then at most 
\begin{equation}
2\sum_{i=1}^m \nu_i
\label{eq:general}
\end{equation}
individuals can have a positive crowding distance. 
%
%\merk{Would prefer to omit the following new paragraph unless a reviewer clearly asked for it}
%\wz{Consider a general case where the Pareto front is $\omega(n)$ for a constant $m$. In order for a good approximation or coverage, one may utilize a population size as a constant fraction of the Pareto front. With (\ref{eq:general}), for the case when each objective has $O(n)$ kinds of values, we see $O(n)$ individuals with positive crowding distance. Thus almost all survival decisions are taken randomly for a sufficiently large $n$, and then we conjecture a bad approximation or coverage of the Pareto front.}  
We omit the details and continue with the heart of this paper, an exponential lower bound of the \NSGA on \momm for $m \ge 4$. 

\begin{theorem}
  Let $m \in \Z_{\ge 4}$ and $a > 1$ be constants. Consider the \NSGA with an arbitrary way to generate the offspring population and with either the original survival selection or the sequential survival selection regarded in~\cite{KukkonenD06,ZhengD22gecco}. Assume that this algorithm optimizes the \momm benchmark with problem size $n$, using a population size of at most $N \le aM$, where $M = (2n/m+1)^{m/2}$ is the size of the Pareto front of \momm. Then with probability at least $1-T\exp(-\Omega(M))$, for the first $T$ iterations at least a constant fraction of Pareto front is not covered by $P_t$. In particular, the time to compute the full Pareto front is exponential in $n^{m/2}$ both in expectation and with high probability.
\label{lem:expruntime}
\end{theorem}

\begin{proof}
Since we aim at an asymptotic results, we can assume that $n$ is at least $10m$ and large enough to ensure that $a \le M/8$. Let $R_t$ be the combined parent and offspring population in iteration~$t$. Let $M'$ denote the size of the part of the Pareto that is front covered by $R_t$, that is, $M'=|f(R_t)|$ when viewing $f(R_t)$ as set, not as multiset. Let $\gamma$ be a constant in $[4/5,1)$. If $M'<\gamma M$, then with probability one, the next generation will cover at most $M'<\gamma M$ Pareto front points as the survival selection cannot increase the coverage of the Pareto front. 

Now we consider the case $M'\ge \gamma M$. 
%Let $c \ge 1$ be one constant such that there exists a constant $c'\in(0,1)$, $M'-|R_t|/(ac) > c'M$. Obviously, such $c$ exists as $|R_t|=2N\le 2aM$ and $M'=\Omega(M)$. Let $a'=ac$, then 
We note that there are at most $|R_t|/(4a\gamma)$ Pareto front points with more than $4a\gamma$ corresponding individuals in~$R_t$ as otherwise there would be more than $(|R_t|/(4a\gamma))(4a\gamma)=|R_t|$ individuals in $R_t$. Let $U$ denote the set of Pareto front points that have at most $4a\gamma$ corresponding individuals in $R_t$, and let $U'=\{u\in U \mid \forall x\in R_t : f(x)=u \Longrightarrow \cDis(x) =0 \}$. Then, using $R_t=2N\le 2aM$, we have
\begin{align*}
|U| \ge M' - \frac{|R_t|}{4a\gamma} \ge \gamma M - \frac{2aM}{4a\gamma} =\frac\gamma2 M.
\end{align*}
%$|U'| \ge M'-|R_t|/(4a\gamma)>c'M$.

Let $A$ denote the number of individuals in $R_t$ with positive crowding distance. From Lemma~\ref{lem:poscDis} we know that $A\le 4n+2m$. Thus we have
\begin{align*}
|U'|\ge |U|-A\ge \frac\gamma2 M-(4n+2m) \ge \frac{\gamma}{4} M,
\end{align*}
where the last inequality uses $4n+2m \le \frac\gamma4 (2n/m+1)^{m/2}$ for $n$ sufficiently large. Besides,
our assumptions, most notably $n \ge 10m$, imply that $N \ge A$. More detailedly, we compute
\begin{equation}
\begin{split}
N&={}\tfrac12 |R_t| \ge \tfrac12 M' \ge \tfrac\gamma2 M=\tfrac\gamma2(\tfrac{2n}{m}+1)^{m/2}\\
&\ge{} \tfrac25\left(1+\tfrac{2n}{m}\left(\tfrac  m2 -1\right)\right)\left(\tfrac{2n}{m}+1\right)\\
&\ge{} \tfrac25\left(1+\tfrac{2n}{m}\tfrac  m4\right)\left(\tfrac{2n}{m}+1\right)\\
&\ge{}\tfrac15 n\left(\tfrac{2n}{m}+1\right) \ge 2m\left(\tfrac{2n}{m}+1\right) =4n+2m \ge A,
\end{split}
\label{eq:NA}
\end{equation} 
where the second line uses $\gamma\ge 4/5$ and $(1+x)^y \ge 1+xy$ for any $x\ge -1$ and $y\ge 1$, the third line uses $m\ge 4$, and the last line uses $n\ge 10m$. 
Thus the number of the individuals with zero crowding distance $2N-A$ is greater than or equal to $N$, the number of individuals to be removed in the survival selection. Since all individuals in $R_t$ are Pareto optima and thus $F_{i^*}=F_1$, the original survival selection will remove $N$ individuals out of the $2N-A$ individuals with zero crowding distance uniformly at random. For the sequential survival selection, it is not difficult to see that removing an individual with zero crowding distance does not change the crowding distances of all other individuals. Consequently, the selection process is identical to the original survival selection.
%\merk{maybe argue that $2N-A \ge N$ and why the survival selection method is not important} 

We thus analyze the effect of randomly removing $N$ individuals with zero crowding distance. To ease the argument, we instead consider the process where we $N$ times independently and uniformly at random pick an individual from the $2N-A$ individuals \emph{with replacement} and then remove these individuals. Obviously, the number of uncovered values in $U'$ after this process, denoted by $Y$, is stochastically dominated by that number, denoted by $X$, after the original process. For a certain value in $U'$ with $b$ corresponding individuals in $R_t$, we know that its probability to be uncovered is at least 
\begin{align*}
%\left(1-\left(1-\frac{1}{2N-A}\right)^N\right)^{a'}\ge \left(1-\exp\left(\frac{N}{2N-A}\right)\right)^{a'}:=p.
&{}\binom{N}{b}b!\left(\frac{1}{2N-A}\right)^b\left(1-\frac{b}{2N-A}\right)^{N-b} \\
&={} \binom{N}{b}b!\left(\frac{1}{2N-A}\right)^b\left(1-\frac{b}{2N-A}\right)^{\left(\frac{2N-A}{b}-1\right)\frac{(N-b)b}{2N-A-b}}\\
&\ge{} \binom{N}{b}b!\left(\frac{1}{2N-A}\right)^b\exp\left(-\frac{(N-b)b}{2N-A-b}\right)\\
&\ge{} \binom{N}{b}b!\left(\frac{1}{2N-A}\right)^b\frac1{e^b}
=\frac1{e^b}\frac{N!}{(N-b)!}\frac{1}{(2N-A)^b}\\
&={}\frac{N\cdots(N-b+1)}{e^b(2N-A)^b}\ge\left(\frac{N-b+1}{e(2N-A)}\right)^b\\
&\ge{} \left(\frac{N-4a\gamma+1}{e(2N-A)}\right)^{4a\gamma} \ge \left(\frac1e+\frac{1-4a\gamma}{eN}\right)^{4a\gamma}:=p,
\end{align*}
where the second and the last inequality use $2N-A\ge N$, and the last inequality also uses $N-4a\gamma+1>0$ since $N\ge M'/2\ge \gamma M/2 \ge 4a\gamma$ from $a\le M/8$. 
Obviously, $p=\Theta(1)$. 

We have just seen that $E[Y]\ge p|U'|\ge \tfrac{1}{4}p\gamma M$.
We note that $Y$ is functionally dependent on the independent $N$ picks, and that each pick will influence $Y$ by at most $1$. Hence, we apply the method of bounded differences~\cite{McDiarmid89} and obtain that 
\begin{align*}
\Pr[Y&\le{} \tfrac{1}{8}p\gamma M] \le \Pr[Y\le \tfrac12 E[Y]] \\
&\le{} \exp(-\Omega(E[Y])) =\exp(-\Omega(M)).
\end{align*}
Therefore, with probability at least $1-\exp(-\Omega(M))$, at least $\tfrac{1}{8}\gamma M$ objective values of the Pareto front are not covered after the survival selection. 
%
%Now we consider the modified survival selection. It is not difficult to see that removing one individual with zero crowding distance will not change the crowding distance for all other individuals. Then we know that the above analysis also holds for the modified survival selection with current crowding distance.
%Then this lemma is proven.
\end{proof}
Theorem~\ref{lem:expruntime} shows that for any constant $m\ge 4$, that is, for the polynomial Pareto front size $M$, this exponential lower bound already exists.
We note that unlike for Lemma~\ref{lem:poscDis}, the proof of the result above does not immediately extend to the \mcocz and \mlotz benchmarks. The reason is that for these benchmarks, not all individuals are Pareto optimal. Consequently, the non-dominated sorting of $R_t$ may contain several fronts, and thus the lower-priority fronts could prevent Pareto optimal individuals from entering into the selection competition. While the proof does not immediately extend, we do believe that the intrinsic problems of the crowding distance for more than two objectives persist on these benchmarks and forbid an effective optimization. Our informal arguments are two-fold. On the one hand, the problem that there is only a small number of individuals with positive crowding distance applies to all fronts of the non-dominated sorting. So very likely, there is a front in which the survival decisions are taken mostly at random, and this should lead to problems, possibly even to the population not approaching the Pareto front at all. On the other hand, if the population comes close to covering the Pareto front, then the combined parent and offspring population is likely to have more than $N$ Pareto optimal individuals. In this case, the same difficulties as exploited in the proof above arise. As said, we speculate that the \NSGA has difficulties to optimize \mcocz and \mlotz, but we leave the formal proof of such a statement as future work.

\subsection{Also the Combined Parent and Offspring Population Does Not Cover the Pareto Front}
\label{ssec:PtRt4}

Our main result in the previous subsection was that the population $P_t$ of the \NSGA for an at least exponential time does not cover the full Pareto front. One could be optimistic that this is better for the combined parent and offspring population~$R_t$, which, with its size only by a factor of two larger, then would be an interesting output of the algorithm. We now show that this is not true and that also $R_t$ misses a constant fraction of the Pareto front for a long time. 

Since this result naturally depends on how precisely the offspring are generated, we regard in the following one particular setting, namely that each parent creates exactly one offspring via one-bit mutation. We are very optimistic that all other mutation-based variants of the \NSGA regarded in previous runtime analyses works would admit the same result. For reasons of brevity, we do not go into details. We are also optimistic that the following result would remain true for crossover-based variants of the \NSGA~\cite{BianQ22}, but with the little understanding we currently have on how crossover works in the \NSGA, this is clearly only a speculation. The main argument of our rigorous proof below, and the reason for our optimism that similar results hold for other variants of the \NSGA, is that a fair proportion of the Pareto front points not covered by $P_t$ only have a constant probability of being generated in $Q_t$. Hence in expectation, $Q_t$ will miss a constant fraction of these, and with the large amount of independent randomness, this expectation can be turned into a statement true with probability $1 - \exp(-M)$. This argument is similar to the proof of Lemma~10 of the extended version of~\cite{ZhengLD22} on the arxiv preprint server, however, some adaptations were necessary to cope with the larger population size, the larger number of objectives, and the different structure of the Pareto front. Note that the lower bound in Theorem~\ref{lem:expruntime} stems from the survival selection regardless of the generation of the offspring, and the following theorem indicates the pessimistic aspect even before the survival selection.

\begin{theorem}
    Let $m \in \Z_{\ge 4}$ and $a > 1$ be constants. Consider using the \NSGA with population size $N \le aM = a(2n/m+1)^{n/2}$, fair selection (every parent creates one offspring), and one-bit mutation to optimize \momm with problem size $n$. Let $\eps \in (0,1)$ be a sufficiently small constant. Assume that in some iteration $t$ the population $P_t$ covers at most a $(1-\eps)$ fraction of the Pareto front. Then with probability at least $1-\exp(-\Omega(M))$  the combined parent and offspring population $R_t$ will cover at most a fraction of $1-\tfrac15\eps\exp\left(-\frac{20am(5m-\eps)}{\eps(5m^2-4(5m-\eps))}\right)$ of the Pareto front. 
\label{lem:rtcover}
\end{theorem}

To ease the presentation, we use the following natural definition of a neighbor.

\begin{definition}[Neighbors]
Let $F^*\subseteq [0..n']^m$ be the Pareto front for \momm. For any two Pareto front points $v=(v_1,n'-v_1,\dots,v_{m/2},n'-v_{m/2})$ and $v'=(v'_1,n'-v'_1,\dots,v'_{m/2},n'-v'_{m/2})$ in $F^*$, we call $v$ and $v'$ neighbors, denoted by $v \sim v'$, if there exists an $i\in[1..m/2]$ such that $|v_i-v'_i|=1$ and $v_j=v'_j$ for all $j\ne i$. 
%Since the neighbors $v$ and $v'$ differ in $v_i$ for some $i\in[1..m/2]$, we also call $v$ and $v'$ $i$-labeled neighbors, and denote as $v\sim_i v'$. Further if $v_i<v'_i$ for $i$-labeled neighbors $v$ and $v'$, we call $v_i$ the left $i$-labeled neighbor of $v'_i$ and $v'_i$ the right $i$-labeled neighbor of $v_i$.
\end{definition}

We now prove Theorem~\ref{lem:rtcover}.

\begin{proof}[Proof of Theorem~\ref{lem:rtcover}]
We first show that there are at most $(\eps/5)M$ Pareto front points whose neighbors in total have at least $\lceil 5am/\eps \rceil$ corresponding individuals in $P_t$. 
Formally, for any Pareto front point $v \in F^*$, let $P_v=\{x\in P_t \mid f(x) \sim v\}$ be the individuals whose function values are neighbors of $v$. Let $\Delta=\lceil 5am/\eps\rceil -1 $ and let $F'=\{v\in F \mid |P_v| \ge \Delta+1\}$. Then we show $|F'| \le \tfrac{am}{\Delta+1}M$ (and thus $|F'|\le\tfrac{\eps}{5}M$ by the definition of $\Delta$). If this was not true, then, since each individual is in $P_v$ for at most $m$ different points $v$, we have $m|P_t|\ge \sum_{v\in F'} |P_v|$ and thus 
\begin{align*}
|P_t| \ge \frac{1}{m}(\Delta+1)|F'| > \frac{1}{m}(\Delta+1)\frac{am}{\Delta+1} M = aM,
\end{align*}
which contradicts $N\le aM$.

Let $U$ be the set of uncovered Pareto fronts for $P_t$. We assumed $|U| \ge \eps M$. Let $U_1$ be the subset of $U$ defined by 
\begin{align*}
U_1=\big\{v&{}=(v_1,n'-v_1,\dots,v_{m/2},n'-v_{m/2})\in F\setminus F' \\
\bigm\vert{}&{}  v_i\in \left[\lfloor \tfrac{2}{5m}\eps(n'+1)\rfloor.. n'-\lfloor \tfrac{2}{5m}\eps(n'+1)\rfloor\right],\\
&{} i=1,\dots,m/2 \big\}.
\end{align*}
Then 
\begin{align*}
|U_1| \ge |U| - 2\tfrac2{5m}\eps(n'+1)\tfrac{M}{(n'+1)}\tfrac m2 -\tfrac15\eps M \ge\tfrac25 \eps M.
\end{align*}
In the following, we will show that at most a constant fraction (less than one) of $U_1$ can be generated in the offspring population (and thus in $R_t$) both with high probability and in expectation.

For any $v=(v_1,n'-v_1,\dots,v_{m/2},n'-v_{m/2}) \in U_1$, the probability that a parent from $P_v$ generates $v$ is at most 
\begin{align*}
\max\limits_{1\le i \le m/2}&{}\left\{\frac{v_i+1}{n},\frac{n'-v_i+1}{n}\right\} \\
&\le{} \frac{n'-\lfloor \tfrac{2}{5m}\eps(n'+1)\rfloor+1}{n}
\le \frac{n'- (\tfrac{2}{5m}\eps n'-1)+1}{n}\\
&\le{} \left(1-\frac{2\eps}{5m}\right)\frac{n'}{n}+\frac2n = \left(1-\frac{2\eps}{5m}\right) \frac 2m+\frac 2n \\
&\le{} \left(2-\frac{2\eps}{5m}\right) \frac 2m=\frac{4(5m-\eps)}{5m^2},
\end{align*}
where the last inequality uses $n\ge m$. 
From the definition of $U_1$, we know that there are at most $\Delta$ individuals in $P_t$ that can generate $v$. Thus the probability that $v$ is not generated from $P_t$, that is, remains uncovered in $R_t$, is at least
%\begin{align*}
%\bigg(1-&{}\frac{4(5m-\eps)}{5m^2}\bigg)^{m\Delta}\\
%&\ge{} \left(1-\frac{4(5m-\eps)}{5m^2}\right)^{5am^2/\eps}\\
%&={}\left(1-\frac{4(5m-\eps)}{5m^2}\right)^{\left(\frac{5m^2}{4(5m-\eps)} -1\right)\frac{5am^2}{\eps}\frac{4(5m-\eps)}{5m^2-4(5m-\eps)}}\\
%&\ge{} \exp\left(-\frac{20am^2(5m-\eps)}{\eps(5m^2-4(5m-\eps))}\right):=p,
%\end{align*}
\begin{align*}
\bigg(1-&{}\frac{4(5m-\eps)}{5m^2}\bigg)^{\Delta}\\
&\ge{} \left(1-\frac{4(5m-\eps)}{5m^2}\right)^{5am/\eps}\\
&={}\left(1-\frac{4(5m-\eps)}{5m^2}\right)^{\left(\frac{5m^2}{4(5m-\eps)} -1\right)\frac{5am}{\eps}\frac{4(5m-\eps)}{5m^2-4(5m-\eps)}}\\
&\ge{} \exp\left(-\frac{20am(5m-\eps)}{\eps(5m^2-4(5m-\eps))}\right):=p,
\end{align*}
where the first inequality uses $\Delta\le 5am/\eps$. We note that $p=\Theta(1)$ for $m=\Theta(1)$.

Let $Y$ denote the number of uncovered points for $R_t$. Then $E[Y] \ge p|U_1| \ge \tfrac25 \eps pM$. Note that $Y$ is determined by $N$ independent random choices in the mutation for each parent, and each choice of the mutation can reduce $Y$ by at most $1$ and will not increase $Y$. Hence, from the bounded difference method~\cite[Theorem~1.10.27]{Doerr20bookchapter}, we have $\Pr[Y\le \tfrac15 \eps pM]\le \exp(-\Omega(M))$.
\end{proof}

\subsection{Synthetic MOEAs}

For reasons of completeness, we now brief{}ly analyze the performance of the SEMO and GSEMO on the \momm benchmark. Not surprisingly, the fact that they keep a solution for every non-dominated objective value in the population avoids the problems seen for the \NSGA above. Since there is little evidence that these two algorithms perform well in practice, this observation is more of an academic interest, though. 

Our analysis follows roughly the lines of the proof of the result that the SEMO finds the Pareto front of \cocz in time $O(n^{m+1})$ in~\cite{LaumannsTZ04}. A slightly tighter bound of $O(n^m)$ and, for $m=4$, of $O(n^3 \log n)$, was shown via a very general (and thus slightly technical) method in~\cite{BianQT18ijcaigeneral}. We are optimistic that this method can be applied to our problem as well and then gives comparable bounds, but since (i)~we only aim at showing that the SEMO and GSEMO do not have the difficulties proven for the \NSGA in our main result and (ii)~it is not at all clear how tight any of these bounds is (recall that the Pareto front has size $(2n/m)^{m/2}$ only), we content ourselves with a simple proof of the weaker result. 

For any point $u\in F^*$ with some neighbor covered by the population $P$, the probability to generate such point in one generation is at least $\frac1 {n|P|}$ for the SEMO and at least $\tfrac1{|P|}\left(1-\tfrac1n\right)^{n-1}\tfrac1n\ge \tfrac{1}{en|P|}$ for the GSEMO. Since all covered Pareto front points will stay covered, we know that the expected number of iterations to cover $u$ is $O(n|P|)$. Since there are $M$ Pareto front points, we know that $|P| \le M$ during the process. With $M$ Pareto front points to cover, the expected number of iterations to cover the whole Pareto front is at most $O(nM^2)=O(n(2n/m+1)^{m})$. For a clear comparison, we put this simple result into the following theorem. We do not believe this result to be tight, but it suffices to show that the SEMO/GSEMO does not suffer critically from the larger number of objectives, different from the \NSGA.

\begin{theorem}
Consider using the SEMO/GSEMO to optimize the \momm benchmark with problem size $n$. Then the Pareto front is covered in an expected number of $O(n(2n/m+1)^{m})$ function evaluations.
\label{thm:semo}
\end{theorem}

\section{Ineffectiveness for Three Objectives}\label{sec:3omm}

Since the classic way of obtaining many-objective versions of bi-objective benchmarks only gives benchmarks for even numbers of objective~\cite{LaumannsTZ04}, in Section~\ref{sec:4omm} we could only prove that the \NSGA is ineffective from four objectives on. Since small numbers of objectives are common, it is an interesting question at what point precisely the ineffectiveness starts. For this reason, we now define a three-objective version of \momm and prove that the \NSGA cannot optimize it efficiently as well. Here the lower degree of symmetry of the $3$-\omm problem will need some adjustment compared to the previous section, but the broad line of argument will be the same.  

\subsection{3-\omm}

We define a three-objective \omm benchmark as follows.

\begin{definition}
Let $n\in \N$ be even. The three-objective function 3-$\omm=(f_1,f_2,f_3):\{0,1\}^n\rightarrow \R^3$ is defined by
\begin{align*}
f_1(x) &={}\sum_{i=1}^n(1-x_i), \\
f_2(x) &={}\sum_{i=1}^{n/2} x_i,\\
f_3(x) &={}\sum_{i=n/2+1}^n x_i,
\end{align*} 
for all $x=(x_1,\dots,x_n)\in\{0,1\}^n$.
\end{definition}

We immediately see that the first objective counts the number of zeros in the whole bit string, whereas the second and third objectives count the numbers of ones in the first and second half of it. As for the \momm function defined for even~$m$ in Section~\ref{sec:4omm}, any $x\in\{0,1\}^n$ is a Pareto optimum of 3-\omm. The Pareto front thus is $F^*=\{(n-v_2-v_3,v_2,v_3)\mid v_2,v_3\in [0..n']\}$, where we write $n':= n/2$, and it has the size $M:=|F^*|=(n/2+1)^2$. We note that it would be easy to extend this definition to any odd number $m$ of objectives, but we do not see a huge interest in this. 

\subsection{An Exponential Lower Bound for 3-\omm}\label{ssec:exp3}

We now follow similar arguments in Section~\ref{ssec:exp}. In the discussion after Lemma~\ref{lem:poscDis}, a general bound of $2\sum_{i=1}^mv_i$ (where $v_i$ is the number of objective values of the $i$-th objective) holds for the number of individuals with positive crowding distance. This value is $2(n+1+n/2+1+n/2+1)=4n+6$ for the 3-\omm benchmark. 

\begin{lemma}
Consider a set of pair-wise non-dominated individuals $S$ for the NSGA-II solving 3-\omm with problem size $n$. Then the crowding distance $\cDis(S)$ computed via Algorithm~\ref{alg:cDis} will have at most $4n+6$ positive values.
\label{lem:poscDis3}
\end{lemma}
%
%\begin{proof}
%Consider the sorted list $S_{i.1},\dots,S_{i.{|S|}}$ (Step~\ref{stp:ccDise} in Algorithm~\ref{alg:cDis}) for calculating the component of the crowding distance w.r.t. a certain objective $f_i$. With the similar proof in Lemma~\ref{lem:poscDis}, we know that each $f_i$ value has at most $2$ corresponding individuals that have a positive crowding distance contribution from the $i$-th objective.
%Since $f_1\in[0..n]$ and $f_2,f_3\in[0..n']$, we know that there are at most $(n+1)+2(n'+1)=(n+1)+2(n/2+1)=2n+3$ individuals with positive crowding distance. 
%\end{proof}

Now we establish a result analogous to Theorem~\ref{lem:expruntime}. 
Let $A\le 4n+6$. For a constant $\gamma\in[4/5,1)$, let $N\ge \frac{\gamma}{2}(n/2+1)^2$. Since
\begin{align*}
\tfrac{\gamma}{2}(n/2+1)^2 &\ge{} \tfrac{2}{5}(n/2+1)^2\ge \tfrac{1}{10}n^2+\tfrac{2}{5}n+\tfrac{2}{5}\\ 
&\ge{} 4n+16+\tfrac25 \ge 4n+6
\end{align*}
for $n\ge 40$, we have $N\ge A$. 
Using this in the place of (\ref{eq:NA}) in the proof of Theorem~\ref{lem:expruntime}, and replacing Lemma~\ref{lem:poscDis} by Lemma~\ref{lem:poscDis3} there as well, we obtain a proof of the following theorem.

\begin{theorem}
  Let  $a > 1$ be a constant. Consider the \NSGA with an arbitrary way to generate the offspring population and with either the original survival selection or the sequential survival selection regarded in~\cite{KukkonenD06,ZhengD22gecco}. Assume that this algorithm optimizes the $3$-\omm benchmark with problem size $n$, using a population size of at most $N \le aM$, where $M = (n/2+1)^{2}$ is the size of the Pareto front of $3$-\omm. Then with at least probability $1-T\exp(-\Omega(M))$, for the first $T$ iterations at least a constant fraction of Pareto front is not covered by $P_t$. In particular, the time to compute the full Pareto front is exponential in $n^{2}$ both in expectation and with high probability.
\label{cor:expruntime3}
\end{theorem}

\subsection{Also the Combined Parent and Offspring Population Does Not Cover the Pareto Front}\label{ssec:PtRt3}

In a similar manner as in Theorem~\ref{lem:rtcover}, we observe also for $3$-\omm that also the combined parent and offspring population does not cover the full Pareto front. While similar, Theorem~\ref{lem:rtcover3} cannot simply be obtained from changing $m$ to $3$ in Theorem~\ref{lem:rtcover3}. Some adjustments are needed due to the different structure of the Pareto front.

\begin{theorem}
Let $a\ge 1$ be any constant. Consider using the \NSGA with population size $N\le aM$, fair selection, and one-bit mutation to optimize 3-\omm with problem size~$n$. Let $\eps \in (0,1)$ be a sufficiently small constant. Assume that in some iteration $t$ the population $P_t$ covers at most a $(1-\eps)$ fraction of the Pareto front. Then with probability $1-\exp(-\Omega(M))$, the combined parent and offspring population $R_t$ covers at most a fraction of $1-\tfrac15\eps\exp\left(-\frac{20a}{20-\eps}\right)$ of the Pareto front. 
\label{lem:rtcover3}
\end{theorem}

\begin{proof}
For any Pareto front point $v \in F^*$, let $P_v=\{x\in P_t \mid f(x) \sim v\}$ be the individuals whose function values are neighbors of $v$. Let $\Delta=\lceil 20a/\eps\rceil -1 $ and let $F'=\{v\in F \mid |P_v| \ge \Delta+1\}$. Then we show $|F'| \le \tfrac{4a}{\Delta+1}M$ (and thus $|F'|\le\tfrac{\eps}{5}M$ by the definition of $\Delta$). If this was not true, then, since each individual is in $P_v$ for at most $4$ different points $v$, we have $4|P_t|\ge \sum_{v\in F'} |P_v|$ and thus 
\begin{align*}
|P_t|\ge \frac{1}{4}(\Delta+1)|F'| > \frac{1}{4}(\Delta+1)\frac{4a}{\Delta+1} M = aM,
\end{align*}
which contradicts $N\le aM$.

Let $U$ be the set of uncovered Pareto fronts for $P_t$. We assumed $|U| \ge \eps M$. Let $U_1$ be defined by 
\begin{align*}
U_1=\big\{(&{}n-v_1-v_2,v_1,v_{2})\in U\setminus F' \\
\big\vert{}&{} v_1,v_2\in \left[\lfloor \tfrac{\eps}{10}(n'+1)\rfloor.. n'-\lfloor \tfrac{\eps}{10}(n'+1)\rfloor\right] \big\}.
\end{align*}
Then 
\begin{align*}
|U_1| \ge |U| - 2\tfrac{\eps}{10}(n'+1)\tfrac{M}{(n'+1)}2 -\tfrac15\eps M \ge\tfrac25 \eps M.
\end{align*} 

For any $v=(n-v_1-v_2,v_1,v_{2}) \in U_1$, the probability that a parent from $P_v$ generates $v$ is at most 
\begin{align*}
\max\limits_{i \in\{1,2\}}&{}\left\{\frac{v_i+1}{n},\frac{n'-v_i+1}{n}\right\} 
\le \frac{n'-\lfloor \tfrac{\eps}{10}(n'+1)\rfloor+1}{n}\\
&\le{} \frac{n'- (\tfrac{\eps}{10} n'-1)+1}{n}
\le \left(1-\frac{\eps}{10}\right)\frac{n'}{n}+\frac2n \\
&={} \frac12\left(1-\frac{\eps}{10}\right)+\frac 2n \le 1-\frac{\eps}{20},
\end{align*}
where the last inequality uses $n\ge 4$. 
From the definition of $U_1$, we know that there are at most $\Delta$ individuals in $P_t$ that can generate $v$. Thus the probability that $v$ is not generated from $P_t$, that is, remains uncovered in $R_t$, is at least
%\begin{align*}
%\left(1-\frac{\eps}{20}\right)^{4\Delta}&\ge{} \left(1-\frac{\eps}{20}\right)^{40a/\eps} = \left(1-\frac{\eps}{20}\right)^{\left(\frac{20}{\eps}-1\right)\frac{40a}{\eps}\frac{\eps}{20-\eps}} \\
%&\ge{} \exp\left(-\frac{40a}{20-\eps}\right):=p.
%\end{align*}
\begin{align*}
\left(1-\frac{\eps}{20}\right)^{\Delta}&\ge{} \left(1-\frac{\eps}{20}\right)^{20a/\eps} = \left(1-\frac{\eps}{20}\right)^{\left(\frac{20}{\eps}-1\right)\frac{20a}{\eps}\frac{\eps}{20-\eps}} \\
&\ge{} \exp\left(-\frac{20a}{20-\eps}\right):=p.
\end{align*}

Let $Y$ denote the number of uncovered points for $R_t$. Then $E[Y] \ge p|U_1| \ge \tfrac25 \eps pM$. Note that $Y$ is determined by $N$ independent random choices in the mutation for each parent, and each choice of the mutation can reduce $Y$ by at most $1$ and will not increase $Y$. Hence, from the bounded difference method~\cite[Theorem~1.10.27]{Doerr20bookchapter}, we have $\Pr[Y\le \tfrac15 \eps pM]\le \exp(-\Omega(M))$.
\end{proof}

\subsection{Synthetic MOEAs}
With the same arguments as for Theorem~\ref{thm:semo}, we easily obtain the following result that the SEMO/GSEMO easily optimizes also 3-\omm.

\begin{theorem}
Consider using the SEMO/GSEMO to optimize the 3-\omm benchmark with problem size~$n$. Then the Pareto front is covered in an expected number of $O(n^5)$ function evaluations.
\label{thm:semo3}
\end{theorem}

\section{Experiments}

In this section, we present a few experimental results illustrating and complementing our theoretical results. 

We concentrate on the $4$-\omm problem since it is more symmetric than the $3$-\omm problem, but still allows for good visualizations. We note that the two objectives counting zeroes are fully determined by the two corresponding objectives that count ones. Consequently, displaying only these two objectives allows us to present the full information in a two-dimensional graph. 
We use the problem size $n=40$, which gives a Pareto front of reasonable size $M=441$, but is small enough to keep the plots readable. 

Since we understand much better the \NSGA without crossover, we use the \NSGA with mutation as the only variation operator. We use fair selection for reproduction (each parent creates one offspring), standard bit-wise mutation with mutation rate $\frac 1n$ as mutation operator, and we select the next parent population as in the original \NSGA paper~\cite{DebPAM02}. Since any solution of the \momm problem is Pareto optimal, all solutions of the combined parent and offspring population lie on the first front and thus the selection is solely based on the crowding distance, breaking ties randomly. We note that the original definition of the \NSGA also used the crowding distance together with tournament selection for the selection for reproduction. Since our theoretical results showed that almost all individuals have the same crowding distance (of zero), we did not see much advantage for this approach here and preferred fair selection as a mean to reduce the variance. Theorem~\ref{lem:expruntime} showed difficulties of the \NSGA for any population size that is a constant factor larger than the Pareto front size~$M$, so we used the population sizes $4M$, $16M$, $64M$, and $256M$. Noting that the largest of these population sizes, $256M = 112{,}896$ is quite large and we conduct runs with $1000$ iterations, we only conducted $10$ independent repetitions of all \NSGA runs. Since the variances observed were always relatively small, this appeared sufficient. 

For a first impression of the optimization behavior of the \NSGA, we display in Figure~\ref{fig:coveredsize} the coverage of the Pareto front in the first 1,000 iterations of exemplary runs. As can be seen in all four plots, the coverage of the Pareto front after some initial gains quickly reaches a stagnation point and then oscillates around this point. This fits well to our theoretical results. We see that the coverage of the Pareto front slightly improves with growing population size. We determined the average coverage after $1000$ iterations in $10$ runs (Figure~\ref{fig:overall}). The slow increase of the coverage shows clearly that increasing the population size gives only very small improvements. From the roughly linear increase (w.r.t. to the log-scale population size) one could speculate that a population size exponential in the Pareto front size is necessary to cover the whole Pareto front in $1000$ iterations. %While this is speculative, nevertheless this figure indicates that the \NSGA will not only for linear population sizes fail to find the full Pareto front in a reasonable time. 

In contrast, as predicted by Theorem~\ref{thm:semo}, the GSEMO algorithm computes the Pareto front very efficiently. In $30$ independent runs, the GSEMO found the full Pareto front on average in $8.84\times 10^4$ iterations (standard deviation $2.97\times10^4$). We recall that the GSEMO evaluates one individual per iteration. Hence these $8.84\times 10^4$ fitness evaluations are significantly less than, e.g., the $1000 \cdot 4M = 1.764 \times 10^6$ fitness evaluations leading to no coverage of the Pareto front in the first chart of Figure~\ref{fig:coveredsize}. 

\begin{figure}[!ht]
\centering
\includegraphics[width=2.85in]{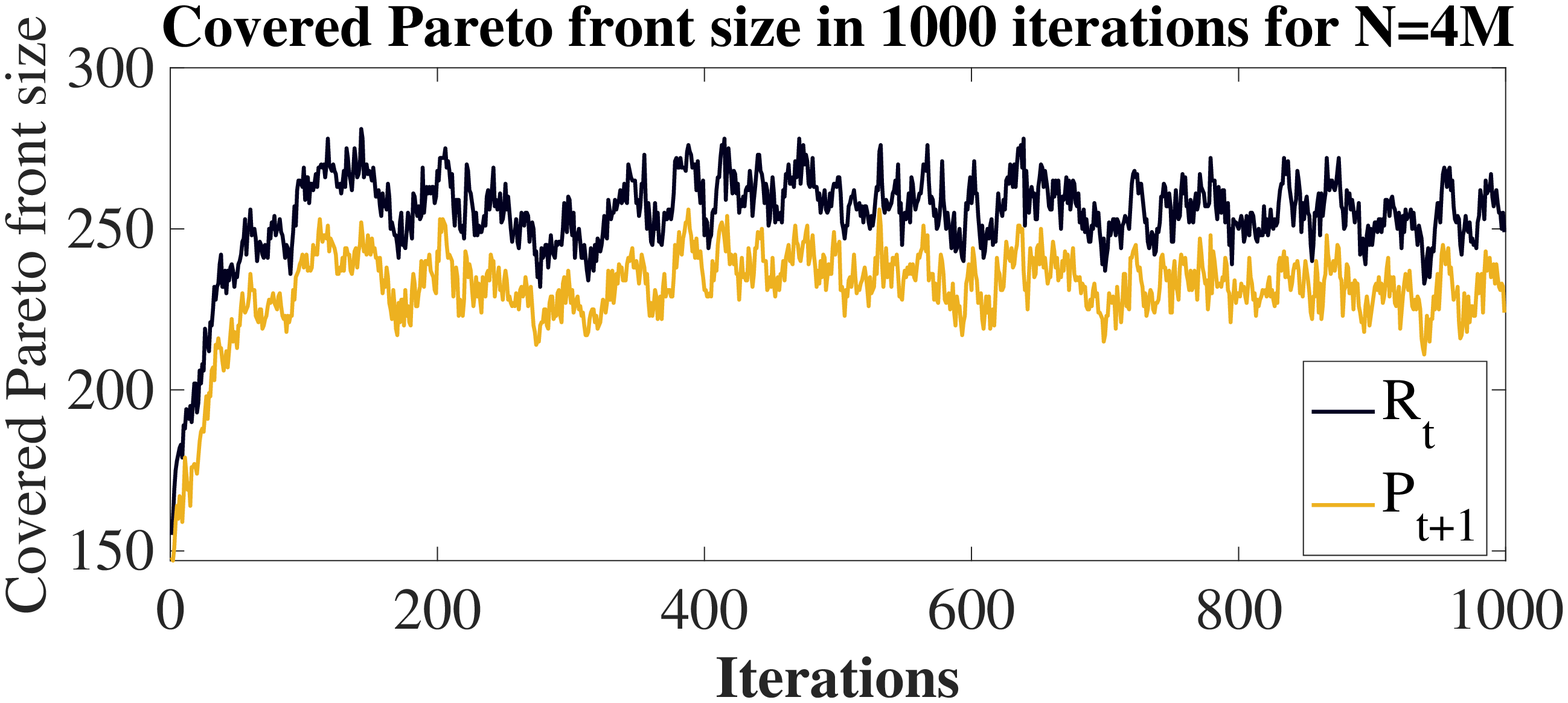} % Reduce the figure size so that it is slightly narrower than the column. Don't use precise values for figure width.This setup will avoid overfull boxes.
\includegraphics[width=2.85in]{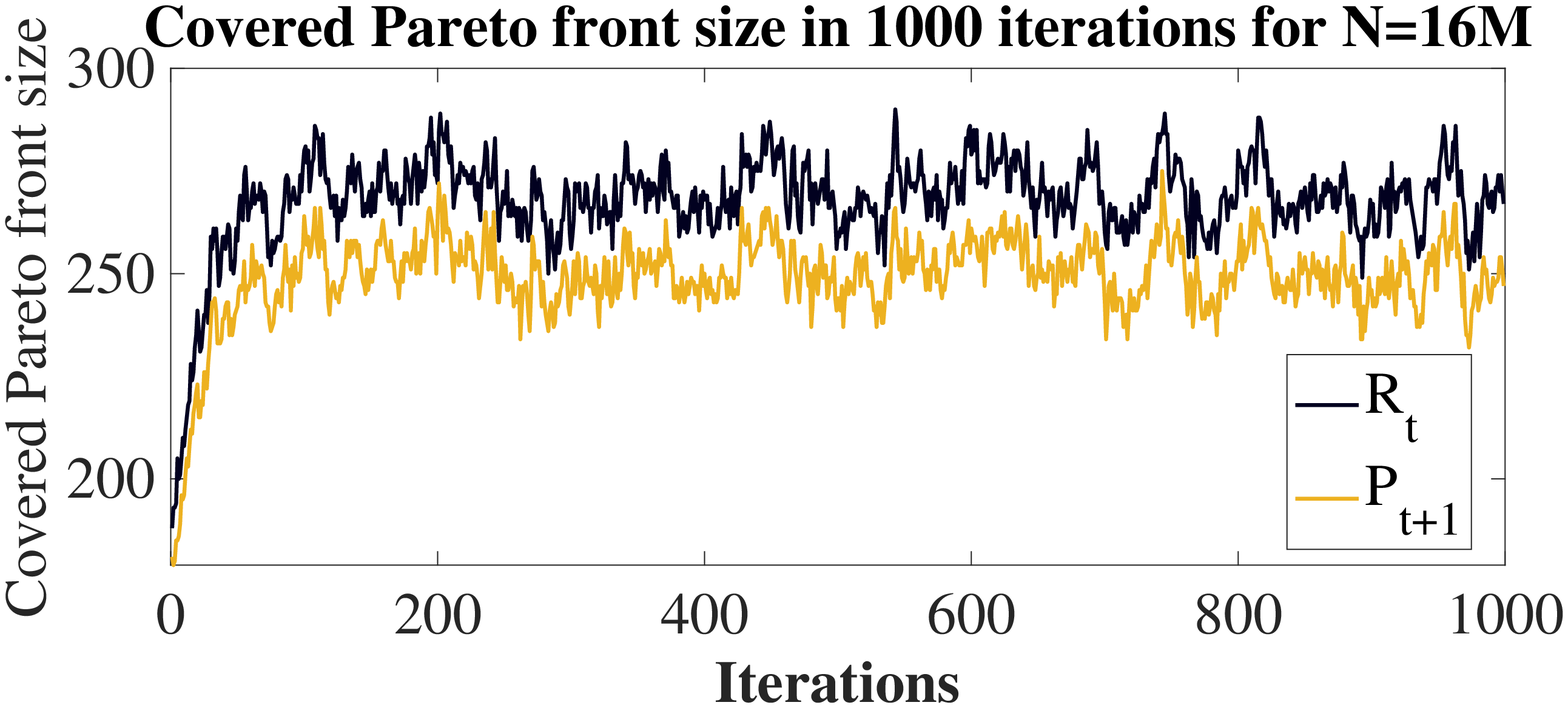} % Reduce the figure size so that it is slightly narrower than the column. Don't use precise values for figure width.This setup will avoid overfull boxes.
\includegraphics[width=2.85in]{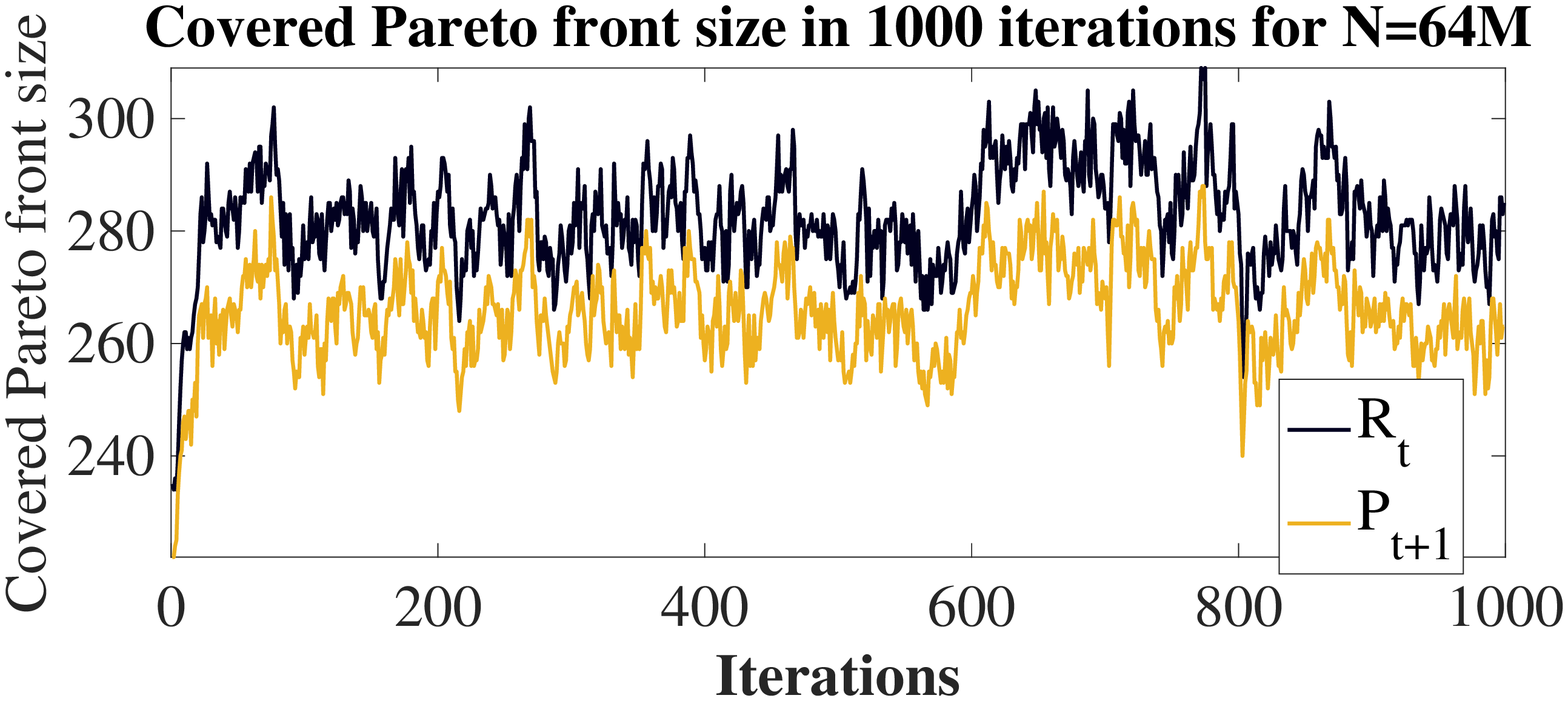} % Reduce the figure size so that it is slightly narrower than the column. Don't use precise values for figure width.This setup will avoid overfull boxes.
\includegraphics[width=2.85in]{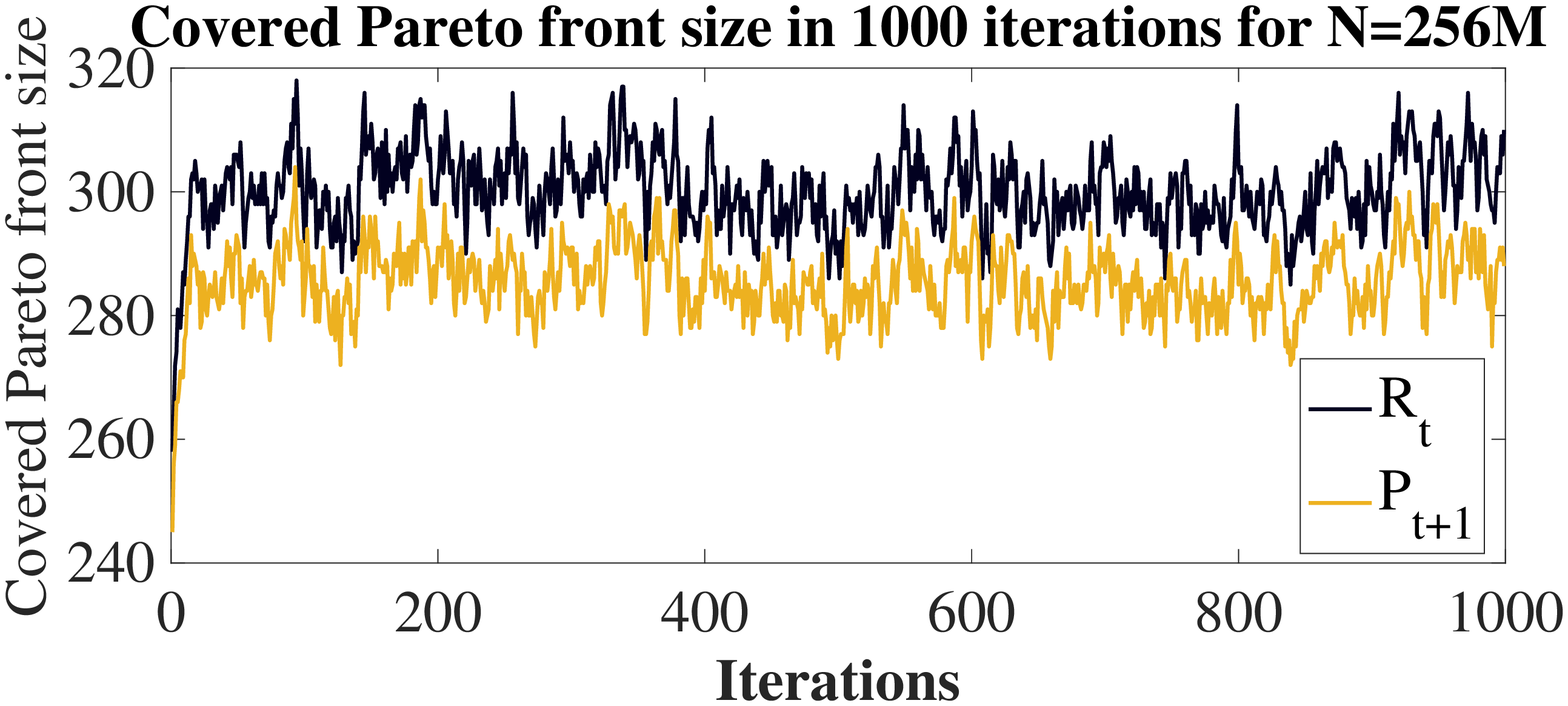} % Reduce the figure size so that it is slightly narrower than the column. Don't use precise values for figure width.This setup will avoid overfull boxes.
\caption{Size of the part of the Pareto front covered by the parent population~$P_t$ and the combined parent and offspring population $R_t$ in the first 1000 iterations of exemplary runs of the \NSGA on the $4$-\omm problem with problem size~$n=40$.}
\label{fig:coveredsize}
\end{figure}

To gain some insight into what happens in a single iteration, in Figure~\ref{fig:coveredF2F4} we display (in the objective space, only regarding the second and fourth objective as discussed earlier) the combined parent and offspring population after $1000$ iterations, the individuals with positive crowding distance, and the selected next parent population. The sizes of these populations naturally fit to the previous experiments, but what is new is the shape of the populations, which covers a relatively compact region around the center of the Pareto front. This is an insight that was not detected in our mathematical analysis. We believe that this fact strongly relies on the particular objective function, in contrast to the key argument in the proof of Theorem~\ref{lem:expruntime}, the negative effect of the random selection of the individuals with crowding distance zero, which depended very little on the particular objective function. For these reasons, we did not try to capture this observation via mathematical means. We suspect that with some effort this would be possible, and it might be a way to prove that even much larger population sizes are not sufficient to find the Pareto front of the \momm problem. We also see that the individuals with positive crowding distance appear randomly distributed, which fits again to our mathematical intuition. %When, say, a particular objective value of the second objective appears a large number of times in the population, here for example the value $10$, then a random choice of two of these individuals will have a positive crowding distance with respect to the second objective and all others will have the contribution zero. 

To check that a similar behavior is observed also by the 3-\omm benchmark, we conducted one small experiment on problem size $n=40$, with a population size of $N=16M$, over $10$ independent runs. It showed that the size of the Pareto front covered by the combined parent and offspring population $R$ in the $1,000$-th iteration has a mean value of $291$ (standard deviation of $7.36$). This is again significantly smaller than the Pareto front size of (again) $441$, and indicates that also in experiments there is no huge difference between the cases $m=3$ and $m=4$.

Overall, these experimental results fit well to our mathematical results and exhibit a few interesting details which our proof could not detect.

\begin{figure}[!ht]
\centering
\includegraphics[width=5in]{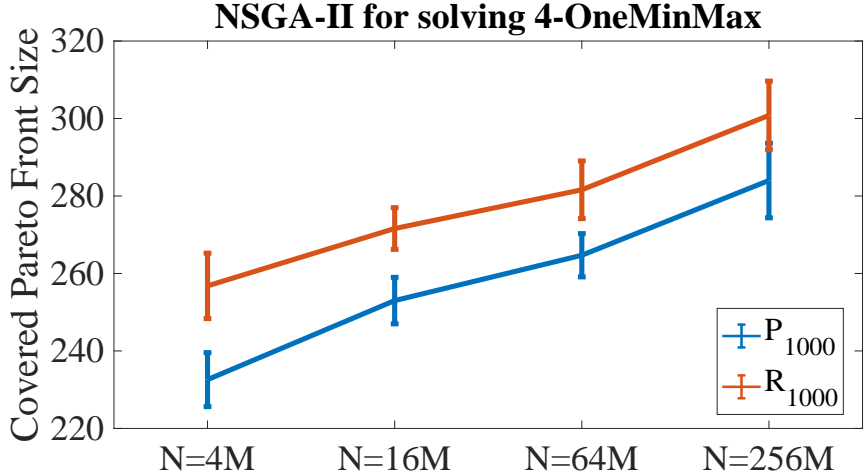} % Reduce the figure size so that it is slightly narrower than the column. Don't use precise values for figure width.This setup will avoid overfull boxes.
\caption{The size of the Pareto front covered by $P_t$ and $R_t$ (mean value $\pm$ standard deviation) in the $1,000$-th iteration for the \mbox{NSGA-II} (fair selection, standard bit-wise mutation) with population size $N=4M,16M,64M,256M$ on \momm with problem size $n=40$ and $4$ objectives, where $M=441$ is the Pareto front size (10 independent runs).}
\label{fig:overall}
\end{figure}

\begin{figure}
\centering
\includegraphics[width=2.75in]{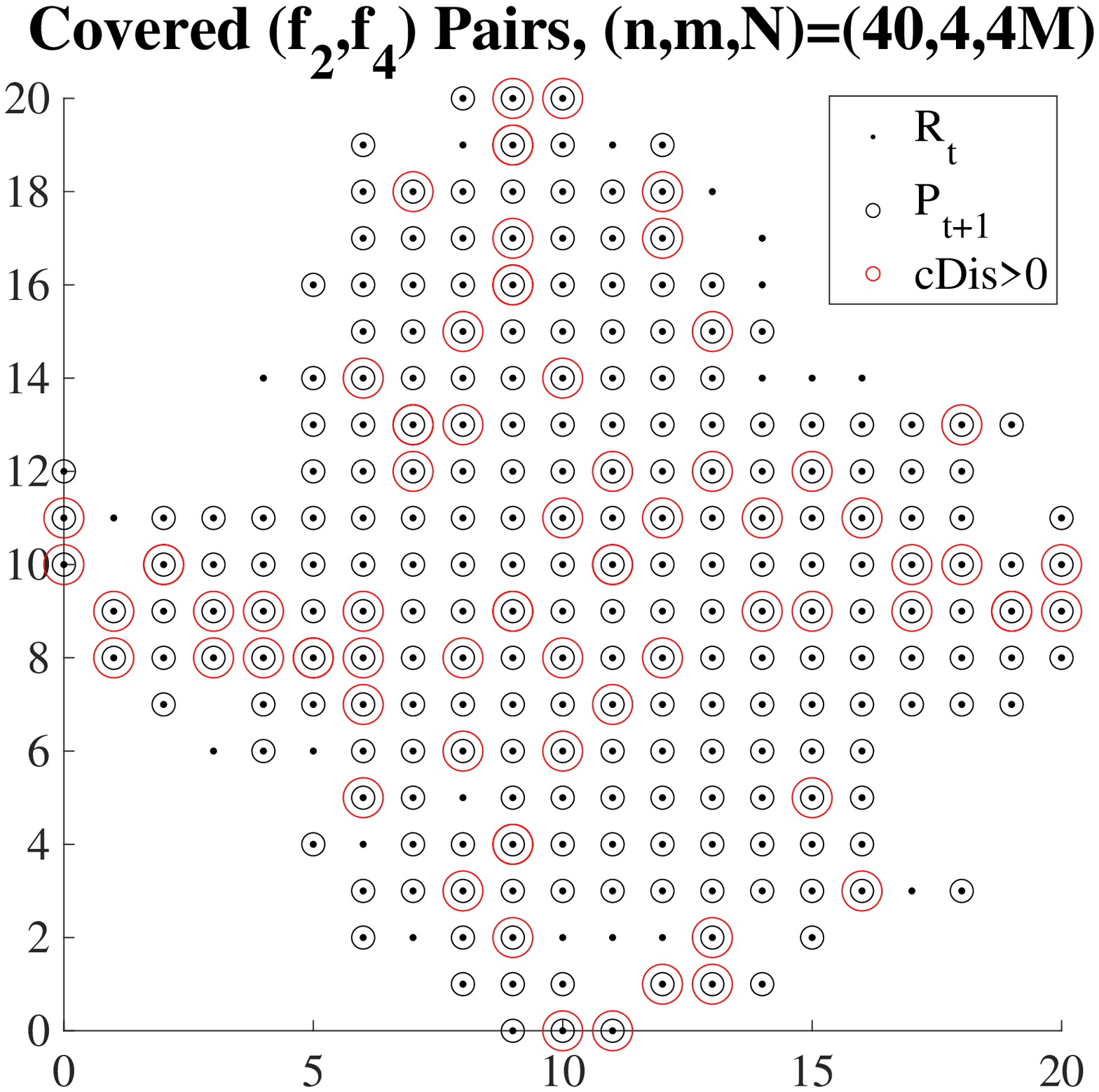} % Reduce the figure size so that it is slightly narrower than the column. Don't use precise values for figure width.This setup will avoid overfull boxes.
\includegraphics[width=2.75in]{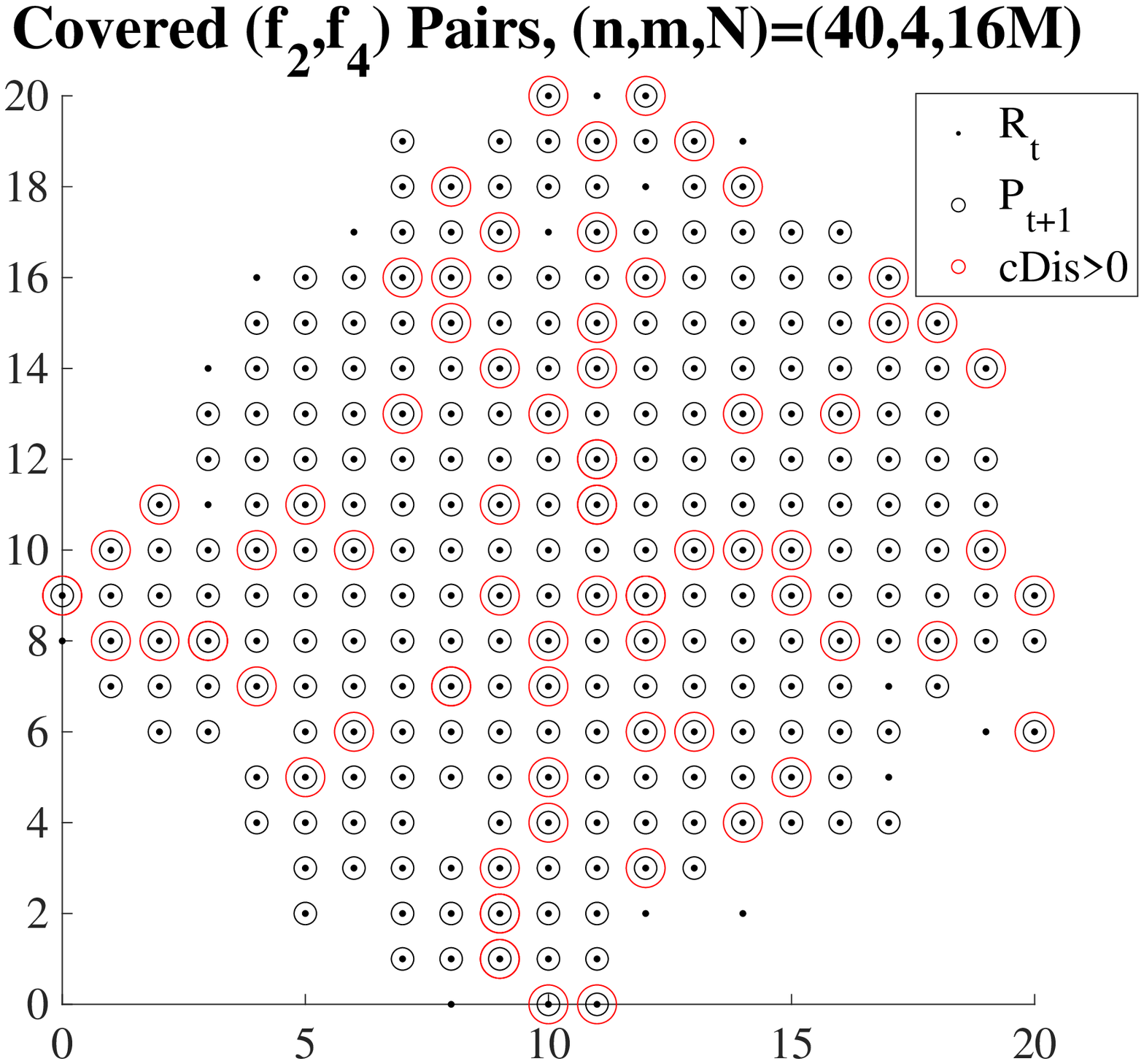} % Reduce the figure size so that it is slightly narrower than the column. Don't use precise values for figure width.This setup will avoid overfull boxes.
\includegraphics[width=2.75in]{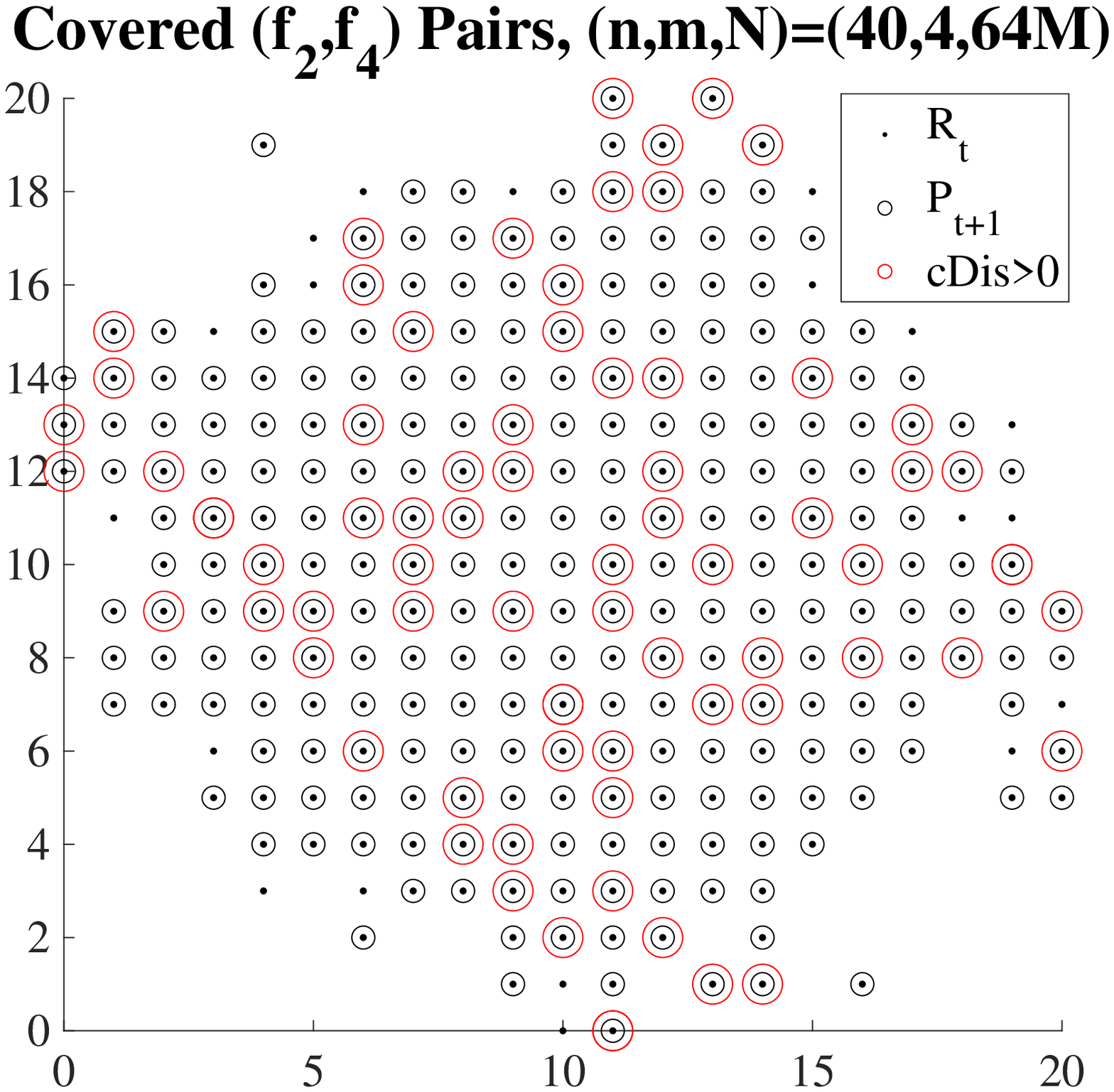} % Reduce the figure size so that it is slightly narrower than the column. Don't use precise values for figure width.This setup will avoid overfull boxes.
\includegraphics[width=2.75in]{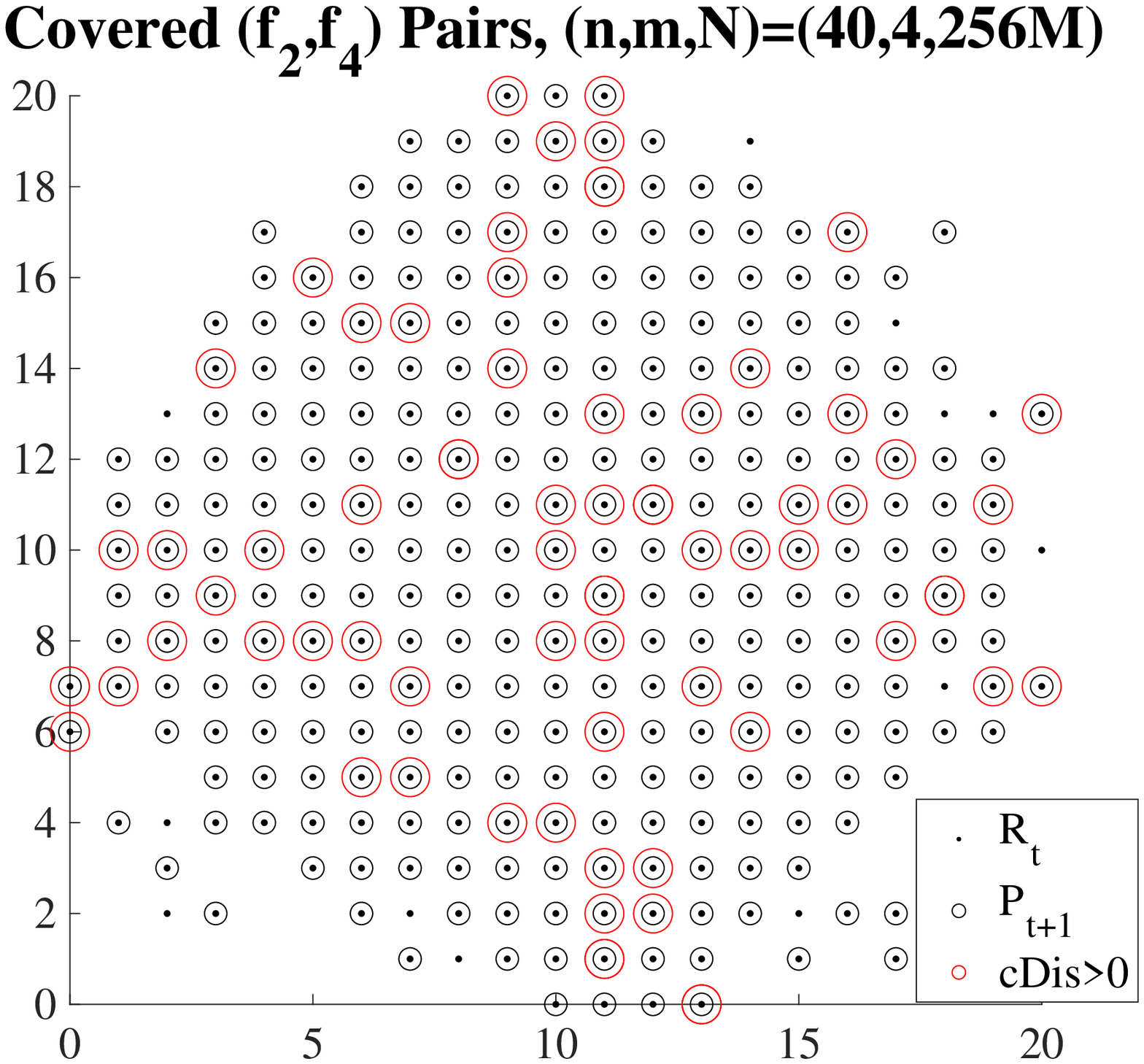} % Reduce the figure size so that it is slightly narrower than the column. Don't use precise values for figure width.This setup will avoid overfull boxes.
\caption{Covered $(f_2,f_4)$ objective values before (the combined parent population $R_t$ as well as the individuals in $R_t$ with positive crowding distance) and after (the next population $P_{t+1}$) the original survival selection of the \mbox{NSGA-II} (fair selection, standard bit-wise mutation) with population size $N=4M,16M,64M,256M$ on \momm with problem size $n=40$ and $4$ objectives where $M=441$ is the Pareto front size. Displayed is one typical run.}
\label{fig:coveredF2F4}
\end{figure}

\section{Conclusion}

In this work, we conducted the first mathematical runtime analysis of the \NSGA on many-objective problems. It confirms in a very strong manner the empirical knowledge that the \NSGA is less effective for larger numbers of objectives: Already for the very simple \momm benchmark, for which any solution is Pareto optimal, the population of the \NSGA will miss a constant fraction of the Pareto front for at least exponential time. In this result, the population size can be any constant multiple of the Pareto front size. 

Besides this quantification of the ineffectiveness, our work also exhibits the reason for these difficulties: Since the crowding distance is computed by regarding the different objectives independently, solutions that are far away may still cause a solution to have a very low crowding distance, disfavoring it in the selection stage. This effect can be so strong that almost all solutions have a crowding distance of zero. In this case, essentially a random set of $N$ individuals will be removed in the selection stage, which makes it very probable that desired solutions are lost.

The obvious question arising from this work is whether there is a definition for the crowding distance which (ideally provably) avoids this shortcoming, but which can still be computed efficiently. A second interesting task for future research would be to conduct similar mathematical runtime analyses for other common MOEAs. This work has shown that the \NSGA suffers drastically from three objectives on, but the synthetic GSEMO algorithm does not (ignoring a polynomial increase of the runtime, which cannot be avoided simply because also the Pareto front increases in size). This raises the question to what extent other MOEAs provably suffer from larger numbers of objectives. Empirical work like~\cite{KhareYD03} has shown that larger numbers of objectives are a challenge for many MOEAs, but that they suffer from these in a different manner. Hence a mathematical analysis and asymptotic quantification of these observations would be very interesting.  

We also note that this paper does not discuss the convergence to the Pareto front point as each solution is Pareto optimal for \momm. Some work like~\cite{EmmerichDLK10} pointed out the divergent behavior of the \NSGA and other MOEAs. How heavily these behaviors influence the convergence ability of the \NSGA is still unknown. It would be an interesting future work.

\section*{Acknowledgments}
This work was supported by National Natural Science Foundation of China (Grant No. 62306086), Science, Technology and Innovation Commission of Shenzhen Municipality (Grant No. GXWD20220818191018001), Guangdong Basic and Applied Basic Research Foundation (Grant No. 2019A1515110177).

This work was also supported by a public grant as part of the Investissement d'avenir project, reference ANR-11-LABX-0056-LMH, LabEx LMH.

%\bibliographystyle{alpha}
%\bibliography{ich_master,alles_ea_master,manyobj}
\newcommand{\etalchar}[1]{$^{#1}$}

}%end sloppy
\end{document}